\documentclass{article}
\usepackage{graphicx,url}
\usepackage{epsfig,epsf,listings}
\usepackage{amsmath,amssymb,amsfonts,bbm,mathrsfs}
\usepackage{caption}
\usepackage{float}
\usepackage[version=3]{mhchem}
\usepackage{ulem}
\usepackage{subfigure}
\usepackage{fancyhdr}
\usepackage{epstopdf}
\usepackage[utf8]{inputenc}
\usepackage[utf8]{inputenc}
\usepackage{enumerate}
\usepackage{mathbbol}
\usepackage{dsfont}
\usepackage{multirow,verbatim,color}
\usepackage{algorithm}
\usepackage{algpseudocode}
\usepackage{lineno,hyperref}
\usepackage{graphicx,latexsym,amsthm,times}
\usepackage{epsfig}
\usepackage{fancyhdr}
\usepackage{color,ulem}
\usepackage[english]{babel}
\usepackage{url}
\usepackage{amsfonts}
\usepackage{color}
\usepackage[margin=0.5in]{geometry}
\usepackage{amsfonts,bm}
\usepackage{upgreek}
\usepackage{setspace}
\usepackage{float}
\usepackage{tabularx}
\usepackage{xcolor}
\usepackage{graphics}
\usepackage{hyperref}
\usepackage[numbers]{natbib}
\usepackage{soul}

\newtheorem{theorem}{Theorem}[section]

\newtheorem{proposition}[theorem]{Proposition}

\newtheorem{problem}{Problem}

\newcommand{\Ns}{N_{\rm{state}}}
\newcommand{\Nt}{N_t}

\modulolinenumbers[5]

\renewcommand{\xi}{\overline{x}}

\renewcommand{\thispagestyle}[2]{}

\fancypagestyle{plain}{
        \fancyhead{}
        \fancyhead[C]{first page center header}
        \fancyfoot{}
        \fancyfoot[C]{first page center footer}
}
\begin{document}
\thispagestyle{empty}
\setcounter{page}{0}

\begin{Huge}
\begin{center}
Computer Science Technical Report CSTR-{\tt22} \\
\today
\end{center}
\end{Huge}
\vfil
\begin{huge}
\begin{center}
Azam Moosavi, Razvan Stefanescu, Adrian Sandu
\end{center}
\end{huge}

\vfil
\begin{huge}
\begin{it}
\begin{center}
``{\tt Efficient Construction of Local Parametric Reduced Order Models \\ Using Machine Learning Techniques}''
\end{center}
\end{it}
\end{huge}
\vfil

\begin{large}
\begin{center}
Computational Science Laboratory \\
Computer Science Department \\
Virginia Polytechnic Institute and State University \\
Blacksburg, VA 24060 \\
Phone: (540)-231-2193 \\
Fax: (540)-231-6075 \\ 
Email: \url{rstefan@ncsu.edu} \\
Web: \url{http://csl.cs.vt.edu}
\end{center}
\end{large}

\vspace*{1cm}

\begin{tabular}{ccc}
\includegraphics[width=2.5in]{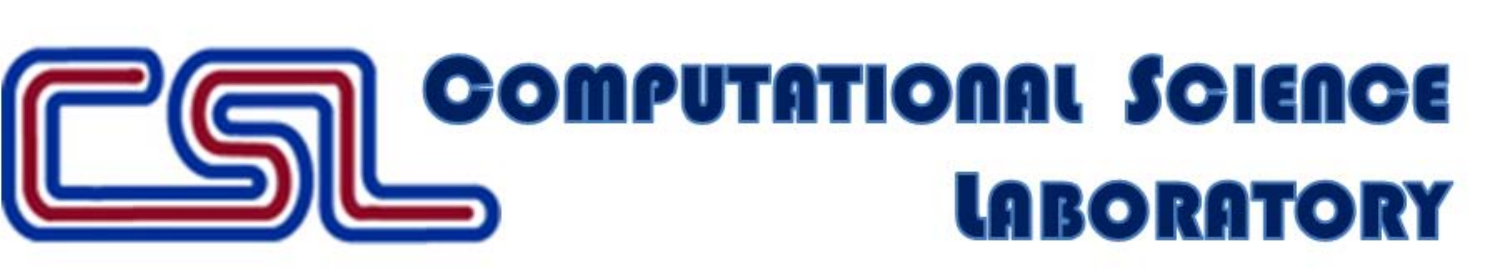}
&\hspace{2.5in}&
\includegraphics[width=2.5in]{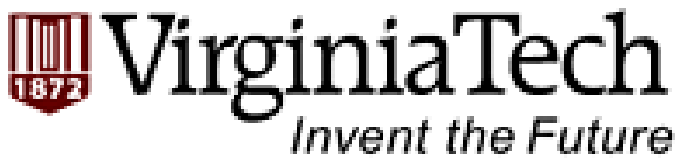} \\
{\bf\em Innovative Computational Solutions} &&\\
\end{tabular}

\newpage


\begin{abstract}
Reduced order models are computationally inexpensive approximations that capture the important dynamical characteristics of large, high-fidelity computer models of physical systems.
This paper applies machine learning techniques to improve the design of parametric reduced order models. Specifically, machine learning is used to develop feasible regions in the parameter space where the admissible target accuracy is achieved with a predefined reduced order basis, to construct parametric maps, to chose the best two already existing bases for a new parameter configuration from accuracy point of view and to pre-select the optimal dimension of the reduced basis such as to meet the desired accuracy. By combining available information using bases concatenation and interpolation as well as high-fidelity solutions interpolation we are able to build accurate reduced order models associated with new parameter settings. Promising numerical results with a viscous Burgers model illustrate the potential of machine learning approaches to help design better reduced order models.

\paragraph{key words}
reduced order models, high-fidelity models, data fitting, machine learning, feasible region of parameters, local reduced order models.
\end{abstract}




\section{Introduction}
\label{sect:Intro}

Many physical phenomena are described mathematically by partial differential equations (PDE),  and, after applying suitable discretization schemes, are simulated on a computer. PDE-based models frequently require calibration and parameter tuning in order to provide realistic simulation results. Recent developments in the field of uncertainty quantification  \cite{le2010spectral,smith2013uncertainty,grigoriu2012stochastic,cacuci2005sensitivity} provide
the necessary tools for validation of such models even in the context of variability and lack of knowledge on the input
parameters. While uncertainty propagation techniques such as Markov chain \cite{isaacson1976markov} and perturbation methods \cite{cacuci2003sensitivity,Cacuci2015687,cacuci2015second} can measure the impact of uncertain parameters on some quantities of interest, they often become infeasible due to the large number of model realizations requirement. Similar difficulties are encountered when solving Bayesian inference problems since sampling from posterior distribution is required.

For large-scale simulations, the variational \cite{sasaki1955fundamental,sasaki1958objective,courtier1987variational,navon1992variational,navon1986review,Stefanescu_lightning_2013} and ensemble \cite{evensen2009data,whitaker2002ensemble,kalnay2003atmospheric,houtekamer1998data} based inference approaches  are widely used in practice. Their efficiency decreases with increasing computational complexity of the underlying physical models. However, increasing model complexity is unavoidable as science fields progress. For example, finer space resolution of the underlying PDE models is one of the most important factors contributing to the one day/decade growth rate of the reliability of atmospheric weather predictions  \cite{buizza2010horizontal,wedi2014increasing}.

The need for computational efficiency motivated the development of surrogate models such as response surfaces, low resolution, and reduced basis models, in order to facilitate optimization, inference, and uncertainty quantification.

Data fitting or response surface models \cite{smith2013uncertainty} are constructed using only a data-driven angle. The underlying physics remains unknown and only the input-output behavior of the model is considered. Data fitting can use techniques such as regression, interpolation, radial basis function, Gaussian Processes, artificial neural networks and other supervised machine-learning methods. The latter techniques can automatically detect patterns in data, and one can use them to predict future data under uncertainty in a probabilistic framework \cite{murphy2012machine}. While easy to implement due to the non-intrusive nature, the prediction abilities may suffer since the governing physics is not specifically accounted for.


Low fidelity models attempt to reduce the computational burden of the high-fidelity models by neglecting some of the physical aspects (e.g., replacing Navier-Stokes and Large Eddy Simulations with inviscid Euler's equations and Raynolds-Averaged Navier-Stokes \cite{gano2005hybrid,sagaut2006large,wilcox1998turbulence}, or
decreasing the spatial resolution \cite{Courtier_Thepaut1994,tremolet2007incremental}). The additional approximations, however, may considerably degrade the physical solution with only a modest decrease of the computational load.

Reduced basis \cite{porsching1985estimation,BMN2004,grepl2005posteriori,patera2007reduced,rozza2008reduced,Dihlmann_2013} and Proper Orthogonal Decomposition \cite{karhunen1946zss,loeve1955pt,hotelling1939acs,lorenz1956eof,lumley1967structure}
are two of the popular reduced order modeling (ROM) strategies available in the literature. Data analysis is conducted to extract basis functions from experimental data or detailed simulations of high-dimensional systems (method of snapshots \cite{Sir87a, Sir87b, Sir87c}), for subsequent use in Galerkin projections that yield low dimensional dynamical models. While these type of models are physics-based and therefore require intrusive implementations, they are usually faster and more robust than data fitting and low-fidelity models.

Robustness of ROM in a parametric setting can be achieved by constructing a global basis \cite{hinze2005proper,prud2002reliable}, but this strategy generates large dimensional bases that may lead to slow reduced order models. Local approaches have been designed for parametric or time domains generating local bases for both the state variables \cite{Rapun_2010,dihlmann2011model} and non-linear terms \cite{eftang2012parameter,peherstorfer2014localized}. {\it Here we address the robustness issue of POD reduced order models by charting the parametric domain with feasible regions of local reduced order models where the reduced solutions are accurate within an admissible prescribed threshold.} Two essential ingredients are used to construct the parametric map: a database of reduced order models, and a data-fitting probabilistic model for the reduced order model errors. Then an incremental procedure uses the newly created probabilistic model to sample the parametric domain and generates a feasible region where a specific reduced order model provides accurate solutions within a prescribed tolerance. We then use a greedy approach to sweep the parameter domain and cover it with such feasible regions. This methodology is applied to the viscous 1D-Burgers model, and a parametric map for the viscosity parameter is generated for various error thresholds. Once the map is constructed there is no need to run the high-fidelity model again, since for each parameter value $\mu^0$ there exists a parameter $\mu$, and the associated reduced order model (basis and reduced operators), in whose interval the current value $\mu^0$ falls; the corresponding reduced solution error is accurately estimated a-priori. The dimension $k$ of the local basis is usually small since it depends only on one high-fidelity model trajectory.

{\it The database of reduced order models and the high-fidelity trajectories are used to statistically model the reduced order approximation errors.} We solve the resulting non-linear regression problems using Gaussian processes (GP) \cite{slonski2011bayesian,lilley2004gaussian} and artificial neural networks (ANN). Specifically, consider the reduced order model of dimension $k$ constructed using the high-fidelity solution computed with parameter value $\mu$. Let $\varepsilon$ be the  error of this reduced model when approximating the full solution at parameter configuration $\mu^0$. We use a GP or ANN approach to model the mapping $\{\mu^0, \mu, k\} \to \log(\varepsilon)$. Our approach is inspired by the multifidelity correction and ROMES methodologies available in the literature for estimation of surrogate models errors using a global basis. Multifidelity correction \cite{alexandrov2001approximation,eldred2004second,gano2005hybrid,huang2006sequential}
has been developed for low-fidelity models in the context of optimization. They simulate the input-output relation $\mu \to \varepsilon$, where $\varepsilon$ is the low-fidelity model errors. The ROMES method \cite{drohmann2015romes} introduces the concept of error indicators for reduced order models and generalizes the 'multifidelity correction' framework by approximating the mapping $\rho(\mu) \to \log(\varepsilon)$. The error indicators $\rho(\mu)$ include rigorous error bounds and reduced order residual norms, while $\varepsilon$ is the reduced order model error at $\mu$ using a reduced global basis. By estimating the log of the reduced order model error instead of the error itself the probabilistic map exhibits a lower variance as shown by our numerical experiments as well as those in \cite{drohmann2015romes}.

The size of a feasible region directly depends on the location of parameter value $\mu$ within the parametric space. {\it For overlapping feasible regions one can combine the available bases and the high-fidelity model trajectories in an efficient manner to generate more accurate reduced order models.} Three different approaches are proposed here, i.e. bases interpolation, bases concatenation, and high-fidelity model solution interpolation. Assuming a linear dependence we perform a Lagrangian interpolation of the bases in the matrix space \cite{lieu2004parameter}, or interpolate their projections onto some local coordinate systems \cite{lieu2004parameter,Amsallem_2008}. Following the idea of the spanning ROM introduced in \cite{weickum2006multi}, for a new parameter not found in the ROM database, we create a projection basis either by concatenating two of the available bases that generated the higher accurate reduced order solutions for the new configurations, or by interpolating the associated high-fidelity solutions and then extracting the singular vectors.

Finally, {\it we address the issue of a-priori selection of the reduced basis dimension for a prescribed accuracy of the reduced solution}. The standard approach is to analyze the spectrum of the snapshots matrix, and use the largest singular value removed from the expansion to estimate the accuracy level \cite{volkwein2007proper}. To take into account the error due to the integration in the reduced space, here we use data fitting models to approximate the mapping  $\{\mu, \log(\varepsilon)\} \to k$. Numerical results obtained using Gaussian processing and artificial neural networks are very promising.

The remainder of the paper is organized as follows. Section \ref{sect:ROM} reviews the reduced order modeling parametric framework. The problems solved herein are formulated in Section \ref{sect:problem_Des}. Gaussian processes and artificial neural networks are reviewed in Section \ref{sect:prob_fram}. The proposed techniques  for generating reduced order bases for new parameter configurations are introduced in Section \ref{sec:data_comb}. Section \ref{sect:experm} presents the details of the data fitting models, constructs the parametric map for the viscous 1D-Burgers model, and analyses the probabilistic model's performance. Conclusions are drawn in Section \ref{sect:conc}.

\section{Parametrized Reduced Order Modeling}\label{sec:ROM}
\label{sect:ROM}

Proper Orthogonal Decomposition has been successfully applied in numerous applications such as compressible flow \citep{Rowley2004}, computational fluid dynamics \citep{Kunisch_Volkwein_POD2002,Rowley2005,Willcox02balancedmodel}, and aerodynamics  \citep{Bui-thanh04aerodynamicdata}, to mention a few. It can be thought of as a Galerkin approximation in the spatial variable built from functions corresponding to the solution of the physical system at specified time instances. A system reduction strategy for Galerkin models of fluid flows leading to dynamic models of lower order based on a partition in slow, dominant, and
 fast modes, has been proposed in \cite{Noack2010}. Closure models and stabilization strategies for POD of turbulent flows have been investigated in \cite{San_Iliescu2013,wells2015regularized}.

In this paper we consider discrete inner products (Euclidian dot product), though continuous products may be employed as well.  Generally, an atmospheric or oceanic computer model is described by the following semi--discrete dynamical system:
\begin{equation}
\label{eqn::-4}
\frac{d{\bf x}(\mu^0,t)}{dt} = {\bf F}({\bf x},t,\mu^0),~~~~{\bf x}(\mu^0,0) = {\bf x}_0 \in \mathbb{R}^{\Ns},
\quad \mu^0 \in \mathcal{\tilde P}.
\end{equation}
The input-parameter $\mu^0$ typically characterizes the physical properties of the flow. For a given parameter configuration $\mu^0$ we select an ensemble of $ N_t $ time instances of the flow ${\bf x}_{t_1}^{\mu^0},...,{\bf x}_{t_{N_t}}^{\mu^0}  \in \mathbb{R}^{\Ns}$, where ${\Ns}$ is the total number of discrete model variables, and $N_t \in \mathbb{N},~N_t>0$.  The POD method chooses an orthonormal basis $U_{\mu^0}=[{\bf u}_{i}^{\mu^0}]_{i=1,..,k}\in \mathbb{R}^{{\Ns}\times k}$, $k>0$, such that the mean square error between ${\bf x}(\mu^0,t_i)$ and the POD expansion ${\bf x}_\textsc{pod}^{\mu^0}(t_i) = U_{\mu^0}{\bf \tilde x}(\mu^0,t_i)$, ${\bf \tilde x}(\mu^0,t_i) \in \mathbb{R}^k$, is minimized on average. The POD space dimension $k \ll {\Ns}$ is appropriately chosen to capture the dynamics of the flow. Algorithm \ref{euclid} describes the reduced order basis construction procedure \cite{stefanescu2014comparison}.

\begin{algorithm}
 \begin{algorithmic}[1]
 \State Compute the singular value decomposition for the snapshots matrix $[{\bf x}_{t_1}^{\mu^0}~~ \dots ~~{\bf x}_{t_{N_t}}^{\mu^0}] = \bar U_{\mu^0} \Sigma_{\mu^0} {\bar V}^T_{\mu^0},$ with the singular vectors matrix $\bar U_{\mu^0} =[{\bf u}_i^{\mu^0}]_{i=1,..,\Ns}.$
 \State Using the singular-values $\lambda_1\geq \lambda_2\geq ...\lambda_n\geq 0$ stored in the diagonal matrix $\Sigma$, define $I(m)=( {\sum_{i=1}^m \lambda_i)/(\sum_{i=1}^{\Ns} \lambda_i})$.
\State Choose $k$, the dimension of the POD basis, such that $ k=\min_m \{I(m):I(m)\geq \gamma\}$ where $0 \leq \gamma \leq 1$ is the percentage of total information captured by the reduced space $\mathcal{X}^k=\textnormal{range}(U_{\mu^0})$. Usually $\gamma=0.99$.
 \end{algorithmic}
 \caption{POD basis construction}
 \label{euclid}
\end{algorithm}

Next, a Galerkin projection of the full model state and equations \eqref{eqn::-3} onto the space $\mathcal{X}^k$ spanned by the POD basis elements is used to obtain the reduced order model:
\begin{equation}\label{eqn::-3}
 \frac{d{\bf \tilde x}(\mu^0,t)}{dt} = U_{\mu^0}^T\,{\bf F}\bigg(U_{\mu^0}{\bf \tilde x}(\mu^0,t),t,\mu^0\bigg),
 \quad {\bf \tilde x}(\mu^0,0) = U_{\mu^0}^T\,{\bf x}(0).
 \end{equation}
The efficiency of the POD-Galerkin technique is limited to linear or bilinear terms, since the projected nonlinear terms at every discrete time step still depend on the number of variables of the full model. In case of polynomial nonlinearities the tensorial POD technique \cite{stefanescu2014comparison}  can be employed to efficiently remove the dependence
on the full dimension by manipulating the order of computions. A considerable reduction in complexity is achieved by the Discrete Empirical Interpolation Method (DEIM) \cite{ChaSor2010,Stefanescu2013}, a discrete variation of Empirical Interpolation Method \cite{MBarrault_YMaday_NDNguyen_ATPatera_2004a}, for any type of nonlinear terms.

While being accurate for the given parameter configuration, the reduced model  \eqref{eqn::-3} loses accuracy when moving away from the initial setting. Several strategies have been proposed to derive a basis that spans the entire parameter space. These includes the reduced basis methods combined with the use of error estimates \cite{rozza2008reduced,quarteroni2011certified,prud2002reliable}, global POD \cite{taylor2004towards,schmit2003improvements}, Krylov-based sampling methods \cite{daniel2004multiparameter,weile1999method}, and  greedy techniques  \cite{haasdonk2008reduced,nguyen2009reduced}. The fundamental assumption used by these approaches is that a smooth low-dimensional global manifold characterizes the model solutions over the entire parameter domain. However, in order to ensure high accuracy of the reduced solution across the parameter space, the dimension of the reduced basis has to be increased in practice, leading to high computational costs. To alleviate this drawback we propose an alternative approach based on local parametric reduced order models.

\section{Problem Description and Solution Methodology}
\label{sect:problem_Des}

 This work addresses the following problems in the construction of reduced order models: designing the parametric map, selecting the best two already existing bases for a new parameter configuration from accuracy point of view, and selecting the optimal dimension of the reduced basis. We formulate them in detail below.

\subsection{Designing the parametric map}

\begin{problem}[Efficient approximate ROMs]
For an arbitrary parameter configuration $\mu^0 \in \mathcal{\tilde{P}}$ construct a reduced order model \eqref{eqn::-3} that provides an accurate and efficient approximation of the high-fidelity solution
\eqref{eqn::-4}:
\begin{equation}\label{eqn::ROM_error_constrain}
  \| {\bf x}(\mu^0,t_i) - U{\bf \tilde x}(\mu^0,t_i) \|_2 < \bar{\varepsilon}, \quad i=1,..,\Nt,
\end{equation}
for some prescribed admissible error level $\bar{\varepsilon} > 0$. The snapshots used to generate the basis $U$ can be obtained with any parametric configuration that belongs to $\mathcal{\tilde{P}}$.
\end{problem}

A simple solution is to solve the high-fidelity model for the specific configuration $\mu^0$, and then build the corresponding reduced order model. However, this approach is computationally expensive.

Our methodology proposes to select a small countable subset $\mathcal{I} = \{\mu_j,~j=1,..,M\} \subset \mathcal{\tilde{P}},~M>0$ and for each $\mu_j$, a reduced order basis $U_{\mu_j}$ along with the reduced operators are constructed for $j=1,..,M$. We denote by $\mathcal{U}$ the set of bases $U_{\mu_j},~j=1,..,M$. Then for each $\mu^0 \in \mathcal{\tilde{P}}$ we can find a basis $U_{\mu_j} \in \mathcal{U}$ and the suitable reduced order model such that its solution satisfies \eqref{eqn::ROM_error_constrain} for $U=U_{\mu_j}$.

This strategy relies on the assumption that a reduced order basis and operators built for a specific parameter configuration $\mu_j \in \mathcal{\tilde{P}}$ can be used to design a reduced order model capable to accurately approximate the solution of the high-fidelity model \eqref{eqn::-4} for all $\mu^0 \in \mathcal{B}(\mu_j,r_j) \cap \mathcal{\tilde{P}}$, where $\mathcal{B}(\mu_j,r_j)$ is the closed ball of radius $r_j \geq 0$ centered at $\mu_j$. Specifically, for $\mu^0 \in \mathcal{B}(\mu_j,r_j) \cap \mathcal{\tilde{P}}$, the reduced order model is constructed by employing the basis $U_{\mu_j}$ and the reduced operators designed at $\mu_j$, i.e.
\begin{equation}
\label{eqn:ROM-from-different-basis}
 \frac{d{\bf \tilde x}(\mu^0,\mu_j,t)}{dt} = U_{\mu_j}^T\,{\bf F}\bigg(U_{\mu_j}\,{\bf \tilde x}(\mu^0,\mu_j,t),t,\mu^0\bigg),\quad {\bf \tilde x}(\mu^0,0) = U_{\mu_j}^T\,{\bf x}(0).
 \end{equation}
Then, by selecting a small radius $r_j$, one can be able to obtain
\begin{equation}
\label{eqn:ROM-from-different-basis-accuracy}
  \| {\bf x}(\mu^0,t_i) - U_{\mu_j}\,{\bf \tilde x}(\mu^0,\mu_j,t_i) \|_2 < \bar{\varepsilon},\quad i=1,..,\Nt,
\end{equation}
for all $\mu^0 \in \mathcal{B}(\mu_j,r) \cap \mathcal{\tilde{P}}$.

The parametric map construction process ends as soon as the entire parameter domain $\mathcal{\tilde{P}}$ is covered with a finite union of overlapping balls $\mathcal{B}(\mu_j,r_j),~j=1,..,M$, corresponding to different reduced order bases and local models
\begin{equation}
\label{eqn:parameter_map}
  \mathcal{\tilde{P}} \subset \bigcup_{j=1}^M \mathcal{B}(\mu_j,r_j),
\end{equation}
such that for each $j=1,2,..,M$ and $\forall \mu^0 \in \mathcal{B}(\mu_j,r) \cap \mathcal{\tilde{P}}$, the solution of the reduced order model \eqref{eqn:ROM-from-different-basis} depending on the basis $U_{\mu_j}$ fulfils the accuracy condition \eqref{eqn:ROM-from-different-basis-accuracy}.

This approach is inspired from the construction of local reduced order models  where the time domain is split in multiple regions \cite{Rapun_2010,peherstorfer2014localized}. In this way the reduced basis dimension is kept small allowing for fast on-line simulations. The cardinality of $\mathcal{I}$ depends inversely proportional with the prescribed level of accuracy $\bar{\varepsilon}$.  As the desired error threshold $\bar{\varepsilon}$ decreases, the map changes since usually the radii $r_j$ are expected to become smaller, and more balls are required to cover the parametric domain, i.e. $M$ is increased.

The construction of the parametric map \eqref{eqn:parameter_map} using the local reduced order models requires the following ingredients:
\begin{enumerate}
  \item The ability to probe the vicinity of $\mu_j \in \mathcal{\tilde{P}}$ and  to efficiently estimate the level of error
\begin{equation}\label{eqn:level_error}
  {\varepsilon} = \max_{i=1,..,N_t}~\| {\bf x}(\mu^0,t_i) - U_{\mu_j}{\bf \tilde x}(\mu^0,\mu_j,t_i) \|_2,
\end{equation}
\item The ability to  find $r_j>0$ such that $\varepsilon \leq \bar{\varepsilon}$ for all $\mu^0 \in \mathcal{B}(\mu_j,r) \cap \mathcal{\tilde{P}}$. This can be theoretically achieved by assuming that the error $\varepsilon$ is monotonically increasing with larger distances $d(\mu^0,\mu_j)$. However, this is not necessarily true and in practice this is obtained by sampling.
  \item The ability to identify the location of a new $\mu_\ell$ (for the construction of a new local reduced order model) given the locations of the previous local parameters $\mu_j$, $j=1,..,\ell-1$, so that
\begin{equation}
\label{eqn:ball_constrain}
 \mathcal{B}(\mu_\ell,r_\ell) \not \subset \bigg( \bigcup_{i=1}^{\ell-1} \mathcal{B}(\mu_i,r_i) \bigg), \quad
  \mathcal{B}(\mu_\ell,r_\ell) \bigcap \bigg( \bigcup_{i=1}^{\ell-1} \mathcal{B}(\mu_i,r_i) \bigg) \neq \emptyset.
\end{equation}
\end{enumerate}

The implementation of the first ingredient does not employ a-posteriori error estimation formulas \cite{nguyen2009reduced}. Inspired from the 'multifidelity correction' and ROMES methodologies we construct probabilistic models to approximate the level of error $\varepsilon$ in \eqref{eqn:level_error}. Gaussian process and artificial neural networks are used to build the probabilistic functions to model the mapping $(\mu^0, \mu_j, k) \to
\log(\varepsilon)$. Since the dimension of basis determines the level of error we include it among the input features. To design probabilistic models with reduced variances we look to approximate the logarithm of the error as suggested in \cite{drohmann2015romes}.

The above machine learning techniques allow to sample the vicinity of $\mu_j$ and estimate the error for each sample parameter value. Based on these error estimates we construct  the ball  $\mathcal{B}(\mu_j,r)$, or perhaps a larger set called a $\mu_j-$feasible region, where the local reduced order model is accurate within the prescribed threshold $\bar{\varepsilon}$.

Next, a greedy algorithm is applied to identify the location of a new parametric configuration $\mu_\ell$ (for the construction of a new basis) depending on the locations of the previous $\mu_i,~i=1,..,\ell-1$. Constraint \eqref{eqn:ball_constrain} is imposed so the entire parametric domain $\mathcal{\tilde{P}}$ satisfies \eqref{eqn:parameter_map} after the map construction is finished.

For the parametric area situated at the intersection of different feasible regions we can assign a new reduced order model based on the information required to construct the initial feasible regions. This is achieved by interpolation or concatenation of the underlying reduced bases or interpolation of the available high-fidelity solutions, as described in detail in Section \ref{sec:data_comb}.

\subsection{Selecting the best two already existing bases for a new parameter configuration}

Since the error of the reduced order solution at a new parameter location $\mu^0$ is not necessarily smaller with the decrease of the distance $d(\mu^0,\mu_j)$, the following more general problem is posed.

\begin{problem}[Selection of best bases]
For a new parameter configuration $\mu^0$ find the best available bases (among the existing ones) that provide the most accurate reduced order model solution.
\end{problem}

The capability of the already proposed probability models can be used to estimate the error
\begin{equation}\label{eqn:level_error2}
  {\varepsilon} = \| {\bf x}(\mu^0,t_i) - U_{\mu_j}{\bf \tilde x}(\mu^0,\mu_j,t_i) \|_2~,i=1,..,N_t,
\end{equation}
for all available $U_{\mu_j},~j=1,2,..$ in the database, and the bases that lead to the smallest estimated errors are selected. This approach is discussed in Section \ref{subsec:select_best_basis}.

\subsection{Selecting the dimension of the reduced basis}

\begin{problem}[Optimal basis dimension]
Find the optimal dimension of the basis $U_{\mu}$ for the parametric configuration $\mu$ such that the error is smaller than the prescribed threshold
\begin{equation}
\label{eqn:level_error3}
  \| {\bf x}(\mu,t_i) - U_{\mu}\,{\bf \tilde x}(\mu,\mu,t_i) \|_2 \le \bar{\varepsilon}, \quad i=1,..,N_t.
\end{equation}
\end{problem}

By optimal we understand the smallest dimension that enables the reduced order model solution to satisfy the error constraint \eqref{eqn:level_error3}. The basis dimension represents one of the most important characteristics of a reduced order model. The reduced manifold size directly affects both the on-line computational complexity of the reduced order model and its accuracy \cite{kunisch2001galerkin,Hinze_Wolkwein2008,fahl2003reduced}. By increasing the size of the basis the projection error decreases and the accuracy of the reduced order model is enhanced. Consequently, the spectrum of the snapshots matrix offers guidance regarding the choice of the reduced basis size when some prescribed reduced order model error is desired. However the accuracy depends also on integration errors in the case of unsteady  models as stated in \cite{homescu2005error}.

We seek to estimate the optimal size of the reduced order model by accounting for both the projection and integration errors. For this we use data fitting models to approximate the mapping $\{\mu,\log\bar\varepsilon\} \to k$, as explained in Section \ref{sect:optimal_base}.

For all the problems addressed in this study a general probabilistic framework is introduced in Section \ref{sect:prob_fram}, along with the description of the supervised machine learning techniques used to construct the discussed probabilistic models.  For each problem we discuss the dataset, the features, as well as the accuracy and stability of the predictions of the associated data fitting models.

\section{Supervised Machine Learning Techniques}
\label{sect:prob_fram}

Consider a random vector ${\bf z}$. Neural networks and Gaussian processes are used to build a probabilistic model $ \phi: {\bf z} \rightarrow \hat{y} $, where $\phi$ is a transformation function that learns through the input feature vector ${\bf z} \in \Omega$ (the sample space) to estimate the deterministic output $y \in \mathbb{R}$ \cite{murphy2012machine}. The estimator $\hat{y}$ is expected to have a low variance. The features of $ {\bf z} $ should be descriptive of the underlying problem at hand \cite{bishop2006pattern}. The accuracy and stability of estimations are assessed using the K-fold cross validation technique. The samples are split into K subsets (``folds''), where typically $3 \le K \le 10$. The machine is trained on $K-1$ sets and tested on the $K$-th set in a round-robin fashion \cite{murphy2012machine}.  Each fold induces a specific error quantified as the average of the absolute values of the differences between the predicted and the $K$-th set values.
\begin{subequations}
\begin{equation}
\label{eqn:err_fold}
\textnormal{E}_{\rm fold}=\frac{\sum_{i=1}^N | \hat{y_i}-y_i | }{N} , \quad
\textnormal{VAR}_{\rm fold}=\frac{\sum_{i=1}^N \left( \hat{y_i} - \textnormal{E}_{\rm fold} \right)^2}{N},
\end{equation}
where {\color{magenta} $N$} is the number of test samples in the fold.  The error is then averaged over all folds:
\begin{equation}
\label{eqn:err_fold_average}
\textnormal{E}=\frac{\sum_{\textnormal{fold}=1}^K\,  \textnormal{E}_{\rm fold}  }{K}, \quad
\textnormal{VAR}=\frac{\sum_{\textnormal{fold}=1}^K \left(\textnormal{E}_{\rm fold}- \textnormal{E} \right)^2}{K}.
\end{equation}
\end{subequations}

The variance of the prediction results \eqref{eqn:err_fold} accounts for the sensitivity of the model to the particular choice of data set. It quantifies the stability of the model in response to the new training samples. A smaller variance indicates more stable predictions, however, this sometimes translates into a larger bias of the model. Models with small variance and high bias make strong assumptions about the data and tend to underfit the truth, while models with high variance and low bias tend to overfit the truth \cite{biasVar_NG} . The trade-off between bias and variance in learning algorithms is usually controlled via techniques such as regularization or bagging and boosting \cite{bishop2006pattern}.

In what follows we briefly review the Gaussian process and Neural network techniques.


\subsection{Gaussian process kernel method}
\label{sect:GP}
A Gaussian process is a collection of random variables, any finite number of which have a joint Gaussian distribution  \cite{rasmussen2006gaussian}. A Gaussian process is fully described by its mean and covariance functions
\begin{equation}
\label{GP_Dist}
\phi(\mathbf{z}) \sim \textnormal{gp}\, \bigl(m(\mathbf{z}), k(\mathbf{z}_i, \mathbf{z}_j ) \bigr),
\end{equation}
where \cite{rasmussen2006gaussian}
\[
\begin{array}{lr} m(\mathbf{z})=\mathbb{E}\left[ \phi(\mathbf{z}) \right], \quad
k(\mathbf{z}_i, \mathbf{z}_j )= \mathbb{E} \left[\left(\phi(\mathbf{z}_i)-m(\mathbf{z}_i)\right) \left( \phi(\mathbf{z}_j)- m (\mathbf{z}_j) \right)  \right].
\end{array}
\]
In this work we employ the commonly used squared-exponential-covariance Gaussian kernel with \cite{rasmussen2006gaussian}:
\begin{equation}
\label{eq_cov}
k(\mathbf{z}_i ,\mathbf{z}_j)=\sigma^2_\phi\, \exp \left(-\frac{\left( \mathbf{z}_i -  \mathbf{z}_j \right)^2}{2\, \ell ^2} \  \right)+ \sigma^2_n  \, \delta_{i,j},
\end{equation}
where $\mathbf{z}_i $ and $\mathbf{z}_j$ are the pairs of data points in training or test samples, and $\delta $ is the Kronecker delta symbol. The model \eqref{eq_cov} has three hyper-parameters. The length-scale $\ell$ governs the correlation among data points. The signal variance $\sigma^2 _\phi$ and the noise variance $\sigma^2 _n $ govern the precision of variance and noise, respectively. 

Consider a set of training data points $\mathbf{z}_1, \mathbf{z}_2, \cdots \mathbf{z}_n $, and the corresponding noisy observations $y_1, y_2, \cdots y_n$,
\begin{equation}
\label{GP_training}
y_i=\phi(\mathbf{z}_i)+ \epsilon_i ,\quad \epsilon_i \sim \mathcal{N} \left(0, \sigma^2_n \right), \quad i = 1,\dots,n.
\end{equation}
%

Consider also the set of test points $\mathbf{z}_1 ^*, \mathbf{z}_2 ^*, \cdots , \mathbf{z}_m ^*$ and the predictions $\hat{y}_1, \hat{y}_2, \cdots \hat{y}_m$
\begin{equation}
\label{GP_test}
\hat{y}_i=\phi\left(\mathbf{z}_i^*\right), \quad i = 1,\dots,m.
\end{equation}
For a Gaussian prior the joint distribution of training outputs $y$ and test outputs $\hat{y}$ is:
\begin{equation}
\label{GP_prior}
\begin{bmatrix}
y\\
\hat{y}
\end{bmatrix}
\sim
\mathcal{N}
\left(
\begin{bmatrix} m(\mathbf{z}) \\ m(\mathbf{z}^*) \end{bmatrix}\, , \,
\begin{bmatrix}
k(\mathbf{z},\mathbf{z})+\sigma^2_nI & k(\mathbf{z},\mathbf{z}^*)\\
k(\mathbf{z}^*,\mathbf{z}) & k(\mathbf{z}^*,\mathbf{z}^*)\\
\end{bmatrix}
\right)
\end{equation}
The predictive distribution represents the posterior after observing the data \cite{bishop2006pattern} and is given by:
\begin{equation}
\label{GP_posterior}
p\left(\hat{y}|\mathbf{z},y,\mathbf{z}^* \right) \sim \mathcal{N}
\left(\, k(\mathbf{z}^*,\mathbf{z})\left( k(\mathbf{z},\mathbf{z})+\sigma^2_nI \right)^ {-1}y\, , \,
k(\mathbf{z}^*,\mathbf{z}^*)- k(\mathbf{z}^*,\mathbf{z})\left( k(\mathbf{z},\mathbf{z})+\sigma^2_nI \right)^{-1} k(\mathbf{z},\mathbf{z}^*) \, \right).
\end{equation}
The prediction of Gaussian process will depend on the choice of the mean and covariance functions, and on their hyperparameters.  The hyperparametrs can be inferred from the data by minimizing the marginal negative log-likelihood function $\theta = \arg\min\, L(\theta)$, where
\[
L(\theta) = - \log\, p(y|\mathbf{z},\theta)=\frac{1}{2} \log |k \left(\mathbf{z},\mathbf{z} \right)|
+ \frac{1}{2}  (y-m(\mathbf{z}))^T\, k \left(\mathbf{z},\mathbf{z} \right)^{-1}\, (y-m(\mathbf{z})) + \frac{n}{2}\, \log \left( 2 \pi \right).
\]
%

\subsection{Artificial Neural networks}
\label{sect:NN}
The study of artificial neural networks (ANNs)  begin in the 1910s in order to intimate human brain's biological structure.
Pioneering work was carried out by Rosenblatt, who proposed a three-layered network structure, the perceptron \cite{hagan2014neural} .
%
ANNs detect the pattern of data by discovering the input--output relationships. Applications include the approximation of functions, regression analysis, time series prediction, pattern recognition, and speech synthesis and recognition \cite{jang1997neuro,ayanzadeh2011fossil}.

The architecture of ANNs is schematically represented in Figure \ref{fig:NN_struct}. ANNs consist of neurons and connections between the neurons (weights).
Neurons are organized in layers, where at least three layers of neurons (an input layer, a hidden layer, and an output layer) are required for construction of a neural network.
The input layer distributes input signals $\mathbf{z}_1, \mathbf{z}_2, \cdots, \mathbf{z}_n$ to the hidden layer. For a neural network with $L$ hidden layers and $m$ neurons in the hidden layer, let $ \hat{y}_j$  be the vector of outputs from layer $\ell$,
$b^\ell$ the biases at layer $\ell$, and $w_{kj}^\ell$ the weight connecting the neuron $j$ to the $k_{th}$ input. Then the feed-forward operation is:
\[
 \begin{array}{lr}
   \mathbf{x}_j^{\ell+1}=\sum_{k=1} w_{kj}^{\ell+1} \hat{y}^{\ell}_k+b_j^{\ell+1},  \quad \hat{y}^0_j= \mathbf{z}_j,\quad j=1,\cdots m \\
   \hat{y}_j^{\ell+1}=\phi \left(\mathbf{x}_j^{\ell+1}  \right), \quad   \ell=0, 1, \cdots, L-1
 \end{array}
\]
The differentiable function $\phi$ is the transfer function and can be log-sigmoid, hyperbolic tangent sigmoid, or linear transfer function.
%
\begin{figure}[H]
	\begin{centering}
	\includegraphics[width=0.35\textwidth, height=0.26\textwidth]{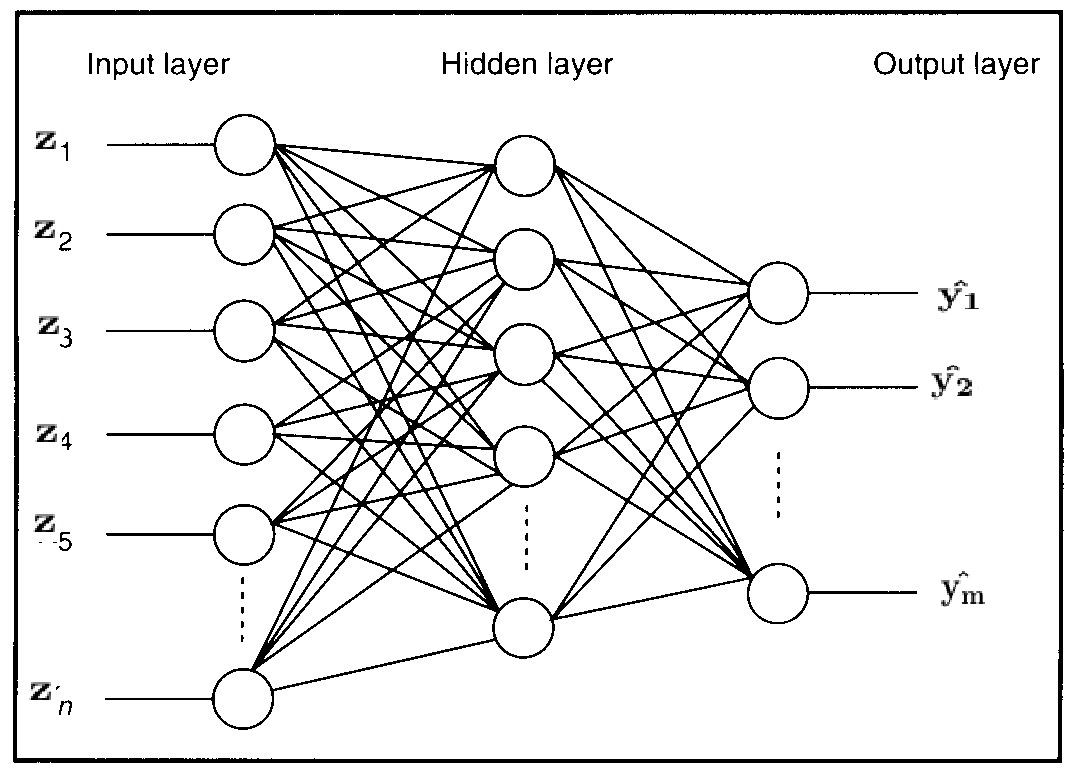}
        \caption{Schematic representation of the neural network structure.
         }
	\label{fig:NN_struct}
	\end{centering}
\end{figure}
%

 The training process of ANN adjusts the weights $w_{i,j}$ in order to reproduce the desired outputs when fed the given inputs. The training process via the back propagation algorithm \cite{rumelhart1985learning} uses a gradient descent method to modify weights and thresholds such that the error between the desired output and the output signal of the network is minimized \cite{funahashi1989approximate}.  In supervised learning the network is provided with samples from which it discovers the relations of inputs and outputs. The output of the network is compared with the desired output,and  the error is back-propagated through the network and the weights will be adjusted. This process is repeated during several epochs, until the network output is close to the desired output \cite{haykin2009neural}. In unsupervised learning a model of the data distribution is formed in order to extract significant data features.The network extracts data features without being shown a set of inputs and outputs \cite{jang1997neuro}.


\section{Combining Available Information for Accurate ROMs at New Parametric Configurations} \label{sec:data_comb}

The POD method produces an orthogonal basis that approximately spans the state solution space of the model for a specific parameter configuration. Moving away from the initial parametric configuration requires the construction of new bases since the initial reduced order model is not accurate anymore. However, if states depend continuously on parameters, the POD basis constructed for one parameter configuration can approximately span the solution space at different parametric settings in a local vicinity.

Several methods to combine the available information to generate more accurate reduced order models for new parameter configurations have been proposed in the literature. One is the interpolation of the available reduced order bases $U_{\mu_j}$ computed for the parametric configurations $\mu_j,~j=1,..,M$. The dependence of the bases on the parameters has been modeled with various linear and nonlinear spatially-dependent interpolants.

Here we compare the performances of different strategies that involve Lagrange interpolation of bases in the matrix space and in the tangent space of the Grassmann manifold. In addition we propose to concatenate the available reduced bases followed by an orthogonalization process, and to interpolate the solutions of the high fidelity model as a mean to derive the reduced order basis for a new parameter configuration.

\subsection{Basis interpolation} \label{sec:Bases_Interpolation}


\paragraph{Lagrange interpolation of bases} Assuming the reduced manifold ${\bf U}: \mathcal{\tilde P} \to \mathbb{R}^{\Ns \times k}$ poses a continuous and linear dependency with respect to the parametric space, and if $M$ discrete bases $U_{\mu_j}={\bf U}(\mu_j)$ have been already constructed for various parametric configurations $\mu_j,~j=1,2,..,M$, then a basis corresponding to the new configuration $\mu^0$ can be obtained using Lagrange's interpolation formula
\begin{eqnarray}
\label{eq:direct_Lagrange_interp}
&& U_{\mu^0} = \sum_{j=1}^M U_{\mu_j}L_j(\mu^0), \qquad
     L_j(\mu) = \prod_{i\neq j} {\frac{\mu-\mu_i}{\mu_j-\mu_i}}.
\end{eqnarray}
Drawbacks of this approach include the orthogonalization requirement for the resulting interpolated basis vectors, and the lack of linear variation in the angles between pairs of subspace planes \cite{lieu2004parameter} spanned by the reduced bases $U_{\mu_j}$. Differential geometry results can be employed to alleviate these deficiencies.

\paragraph{Grassmann manifold}


In the study proposed by Amsallem and Farhat \cite{Amsallem_2008} basis (matrix) interpolation in the tangent space of the Grassmann manifold at a careful selected point $S$ representing a subspace spanned by one of the available reduced bases was performed. It has been shown that Grassmann manifold can be endowed with a differentiable structure \cite{absil2004riemannian,Edelman98thegeometry}, i.e., at each point $S$ of the manifold a tangent space exists. The mapping from the manifold to the tangent space at $S$ is called the logarithmic mapping, while the backward projection is referred to as exponential mapping \cite{begelfor2006affine}. According to \cite{Amsallem_2008} the construction of a new subspace $S_{\mu^0}$ associated with a new parameter $\mu^0$ can be obtained by interpolating the known subspaces $\{S_i\}_{i=1}^{M  }$ spanned by the already computed reduced bases $U_{\mu_i},~i=1,..,M$.
The steps required by this methodology \cite{Amsallem_2008} are described in the Algorithm\ref{alg:Grassmann_manifold}.

\begin{algorithm}
 \begin{algorithmic}[1]

\State Select a point $S_{i_0}$ of the manifold to represent the origin point for the interpolation spanned by the basis $U_{\mu_{i_0}}$.
\State The tangent space $\mathcal{T}_{S_{i_0}}$ and the subspaces $\{S_i\}_{i=1}^{M}$ are considered. Each point $S_{i}$  sufficiently close to $S_{i_0}$
is mapped to a matrix $\Gamma_i$ representing a point of $\mathcal{T}_{S_{i_0}}$ using the logarithm map $\textnormal{Log}_{S_{i_0}}$ \cite{begelfor2006affine}
\begin{eqnarray*}
  (I - U_{\mu_{i_0}} U_{\mu_{i_0}}^T) U_{\mu_i} (U_{\mu_{i_0}}U_{\mu_i})^{-1} &=& R_i \Lambda_i Q_i^T~~ (\textrm{SVD factorization}),\\
  \Gamma_i &=& R_i \tan^{-1}(\Lambda_i) Q_i^T.
\end{eqnarray*}
\State Each entry of the matrix $\Gamma_0$ associated with the target parameter $\mu_0$ is computed by interpolating the corresponding entries of the matrices
$\Gamma_i \in \mathbb{R}^{N_{state} \times k}$ associated with the parameter points $\mu_i,~i=1,..,M-1$. A univariate or multivariate Lagrange interpolation
may be chosen similar with the one introduced in \eqref{eq:direct_Lagrange_interp}.

\State The matrix $\Gamma^0$ representing a point in the tangent space $\mathcal{T}_{S_{i_0}}$ is mapped to a subspace $S^0$ on the Grassmann manifold spanned by
a matrix $U_{\mu^0}$ using the exponential map \cite{begelfor2006affine}

\begin{eqnarray*}
  \Gamma^0 &=& R^0 \Lambda^0 {Q^0}^T~~(\textrm{SVD factorization}), \\
  U_{\mu^0} &=& U_{\mu_{i_0}} Q^0 \cos(\Lambda^0) + R^0\sin(\Lambda^0).
\end{eqnarray*}

\end{algorithmic}
 \caption{Interpolation in a tangent space to a Grassmann manifold algorithm \cite{Amsallem_2008}}
 \label{alg:Grassmann_manifold}
\end{algorithm}

Amsallem and Cortial \cite{Amsallem_Cortial_2007} proved that the subspace angle interpolation \cite{lieu2004parameter,lieu2005pod} is identical to the interpolation in a tangent space to the Grassmann manifold of two reduced-order bases, thus the latter methodology can be viewed as a generalization of the former approach.

\subsection{Basis concatenation}\label{sec:bases_concatenation}

Basis concatenation idea was introduced in \cite{weickum2006multi} and emerged from the notion of a global basis \cite{taylor2004towards,schmit2003improvements}. In the global strategy, the existent high-fidelity snapshots corresponding to various parameter configurations are collected in a single snapshot matrix and then a matrix factorization is performed to extract the most energetic singular vectors. This global basis is then used to build reduced order models for parameter values not available in the initial snapshots set.

Assuming $\mathbf{x}_{\mu_1},~\mathbf{x}_{\mu_2} \in \mathbb{R}^{\Ns\times\Nt}$ are the snapshots corresponding to two high-fidelity model trajectories, the following error estimate holds \cite[Proposition 2.16]{volkwein2007proper}:
\begin{subequations}
\label{eq:Global_POD}
\begin{eqnarray}
 && \bar{\mathbf{x}} = [\mathbf{x}_{\mu_1} \mathbf{x}_{\mu_2}] = {\bar U} \Lambda {\bar V}^T,  \textrm{  (SVD factorization)} \\
 && \| \bar{\mathbf{x}} - \bar U(:,1:k)\, {\bf \tilde x} \|_F = \sum_{i=k+1}^{\Nt} \lambda_i = \mathcal{O}(\lambda_{k+1}),
\end{eqnarray}
\end{subequations}
where $\lambda_i$ is the $i^{\rm{th}}$ singular value of $\bar{\mathbf{x}}$ , $\bar U(:,1:k)$ are the first $k$ singular vectors of $\bar U$ and ${\bf \tilde x} \in \mathbb{R}^{k \times 2\Nt}$. By $\| \cdot \|_F$ we refer to the Frobenius norm.
The snapshot matrix typically stores correlated data and therefore contains linearly dependent columns. For rank deficient snapshot matrices, and in the case where the reduced order bases $U_{\mu_1}$ and $U_{\mu_2}$ corresponding to the trajectories $\mu_1$ and $\mu_2$ are available, we can construct a basis $\hat{U}$ by simply concatenating columns of $U_{\mu_1}$ and $U_{\mu_2}$ such that the accuracy level in \eqref{eq:Global_POD} is preserved.

\begin{proposition}
Consider the following SVD expansions of rank deficient snapshots matrices $\mathbf{x}_{\mu_1},~\mathbf{x}_{\mu_2} \in \mathbb{R}^{\Ns\times\Nt}$
\begin{equation}
 \mathbf{x}_{\mu_j} = U_{\mu_j} \, \Lambda_j\,  V_{\mu_j}^T, \quad j=1,2. \label{eq:SVD_for_snap_matrices}
\end{equation}
There are positive integers $k_1,k_2$, and $\hat{\mathbf{x}} \in \mathbb{R}^{(k_1+k_2) \times 2\Nt}$, such that $\bar{\mathbf{x}}$ defined in \eqref{eq:Global_POD} satisfies
\begin{equation}
  \| \bar{\mathbf{x}} - \hat{U}\; \hat{\mathbf{x}} \|_F = \mathcal{O}(\lambda_{{k}+1}),
\end{equation}
where $\lambda_{k+1}$ is the $(k+1)$-st singular value of snapshots matrix ${\bf \bar x}$, and $\hat{U} = [U_{\mu_1}(:,1:k_1) ~ U_{\mu_2}(:,1:k_2)] \in \mathbb{R}^{\Ns \times (k_1+k_2)}$.

\end{proposition}

\begin{proof}
Since $\mathbf{x}_{\mu_1},~\mathbf{x}_{\mu_2} \in \mathbb{R}^{\Ns\times\Nt}$ are rank deficient matrices, there exist at least two positive integers $k_1$ and $k_2$, such that the singular values associated with $\mathbf{x}_{\mu_1}$ and $\mathbf{x}_{\mu_2}$ satisfy
\begin{equation}
\lambda^1_{k_1+1}, \lambda^2_{k_2+1} \leq \frac{\lambda_{k+1}}{2},~\forall k=0,..,Nt-1.
\end{equation}
Next, from \cite[Proposition 2.16]{volkwein2007proper} and \eqref{eq:SVD_for_snap_matrices} we have the following estimates:
\begin{equation} \label{eq:local_SVD}
   \| \mathbf{x}_{\mu_j} - U_{\mu_j}(:,1:k_j)\, {\bf \tilde x_j} \|_F = \sum_{i=k_j+1}^{\Nt} \lambda_i^j = \mathcal{O}(\lambda_{k_j+1}^j),
\end{equation}
where $\lambda_i^j$ is the $i^{\textrm{th}}$ singular value of $\mathbf{x}_{\mu_j}$ and ${\bf \tilde x_j} \in \mathbb{R}^{k_j \times \Nt}$, for $j=1,2$.

By denoting %
\[
\hat{\mathbf{x}} = \begin{bmatrix}
{\bf \tilde x}_1 \quad {\bf 0}_1 \\
{\bf 0}_2 \quad {\bf \tilde x}_2 \end{bmatrix},
\]
where the null matrix ${\bf 0}_j$ belongs to $\mathbb{R}^{k_j \times \Nt},~j=1,2$, we have
\begin{eqnarray}
&& \| {\bf \bar x} - \hat{U}\, \hat{\mathbf{x}} \|_F = \| [\mathbf{x}_{\mu_1}~~~ \mathbf{x}_{\mu_2}] - [U_{\mu_1}(:,1:k_1){\bf \tilde x}_1~~~ U_{\mu_2}(:,1:k_2){\bf \tilde x}_2]\|_F \leq  \\
&& \| {\bf \bar x}_{\mu_1} - U_{\mu_1}(:,1:k_1){\bf \tilde x}_1\|_F + \| {\bf \bar x}_{\mu_2} - U_{\mu_2}(:,1:k_2){\bf \tilde x}_2\|_F \leq \mathcal{O}(\lambda_{k_1+1}^1) + \mathcal{O}(\lambda_{k_2+1}^2) = \mathcal{O}(\lambda_{k+1}).
\end{eqnarray}
\end{proof}

It is crucial for the proof that $U_{\mu_1}$ and $U_{\mu_2}$ are rank deficient since they have at least one null singular value. In practice usually $k_1+k_2$ is larger than $k$ thus more bases functions are required to form $\hat{U}$ to achieve a similar level of precision as in \eqref{eq:Global_POD} where $\bar U$ is built using a global singular value decomposition. This matrix factorization of $\bar{\mathbf{x}}$ is more costly than both singular value decompositions of matrices $\mathbf{x}_{\mu_j},~j=1,2,$ \eqref{eq:local_SVD} combined for large space dimension $\Ns$. However the off-line stage of the concatenation method also includes the application of a  Gramm-Schmidt-type algorithm to orthogonalize  the overall set of vectors in $\hat{U}$.

For full rank matrices the precision is controlled by the spectra of snapshots matrices $U_{\mu_1}$ and $U_{\mu_2}$ but there is no guarantee that the concatenated basis $\hat{U}$ can provide similar accuracy precision as $\bar U$ in \eqref{eq:Global_POD} for all $k=1,2,..,Nt$.

While the Lagrange interpolation of bases mixes the different energetic singular vectors in an order dictated by the singular values magnitude, this strategy concatenates the dominant singular vectors for each case and preserves their structure.

\subsection{Interpolation of high-fidelity model solutions}

The method discussed herein assumes that the model solution dependents continuously on the parameters. Thus it is natural to consider constructing the basis for a new parameter configuration by interpolating the existent high-fidelity model solutions associated with various parameter settings, and then performing a SVD factorization of the interpolated results. For example, the Lagrange solution interpolant is given by
\begin{eqnarray}
\label{eq:Lagrange_sol_interp}
  \mathbf{x}_{\mu^0} = \sum_{j=1}^M \mathbf{x}_{\mu_j}L_j(\mu^0),
\end{eqnarray}
where $\mathbf{x}_{\mu_j} \in \mathbb{R}^{\Ns\times \Nt}$ is the model solution corresponding to parameter $\mu_j$ and the interpolation polynomials are defined in \eqref{eq:direct_Lagrange_interp}.

A new basis is constructed from the interpolated model solution  \eqref{eq:Lagrange_sol_interp} for the new parametric configuration $\mu_0$. By integrating the corresponding reduced order model the output projected solution will present variations in comparison with the high-fidelity interpolation solution $\mathbf{x}_{\mu^0}$ thus the nonlinear dynamics of the model influence the final solution too.

From computational point of view the complexity of the off-line stage of the solution interpolation method \eqref{eq:Lagrange_sol_interp} is smaller than in the case of the bases concatenation and interpolation approaches. Only one singular value decomposition is required in contrast with the multiple factorizations needed in the latter two strategies where the involved matrices have the same size $\Ns \times \Nt$. Having only $\Nt$ snapshots the size of the outcome basis should be smaller than in the case of basis concatenation approach.

If for each model solution $\mathbf{x}_{\mu_j}$ a singular value decomposition is available such that
\begin{equation}\label{eq:SVD_sim}
   \mathbf{x}_{\mu_j} \approx U_{\mu_j}\, {\bf \tilde x_j},~j=1,2,..,M,
\end{equation}
then from \eqref{eq:Lagrange_sol_interp} we get
\begin{equation}\label{eq:SIM_comp}
   \mathbf{x}_{\mu^0} \approx \sum_{j=1}^M L_j(\mu^0)U_{\mu_j}{\bf \tilde x}_j.
\end{equation}

Now the basis $U_{\mu_0}$ associated with the new configuration is obtained following a matrix factorization, i.e..
\begin{equation}\label{eq:SIM_basis_formula}
   \sum_{j=1}^M L_j(\mu^0)U_{\mu_j}{\bf \tilde x}_j = U_{\mu_0} S_{\mu_0} V^T_{\mu_0}.
\end{equation}

In the basis interpolation method \eqref{eq:direct_Lagrange_interp}, the basis ${U}_{\mu_0}$ (denoted in \eqref{eq:BIM_basis_formula} by ${\bar U}_{\mu_0}$ to differentiate it from the basis described in equation \eqref{eq:SIM_basis_formula}) corresponding to the new parametric setting $\mu_0$ is derived from the following
\begin{equation}\label{eq:BIM_basis_formula}
   \sum_{j=1}^M L_j(\mu^0)U_{\mu_j} = {\bar U}_{\mu_0} \bar S_{\mu_0} \bar V^T_{\mu_0}.
\end{equation}

While there is a close relationship between the philosophy behind the solution interpolation method and the bases interpolation strategy, the algebraic formulations of the bases $U_{\mu_0}$ \eqref{eq:SIM_basis_formula} and ${\bar U}_{\mu_0}$ \eqref{eq:BIM_basis_formula} are significantly different. Consequently an assumption of linear variation of the basis over the parametric interval does not imply that the high-fidelity solution varies linear over the parametric domain. The inverse proposition does not hold neither thus the choice of method depends on the model under study and its input.

\section{Numerical Experiments}
\label{sect:experm}

We illustrate the application of the proposed machine learning methodologies to the construction of error models  for the reduced order models solutions for a one-dimensional Burgers model and their subsequent utilizations. The model proposed herein is characterized by two scalar parameters, but the envisioned parametric map is designed to cover variation in the viscosity coefficient space only. For each of the problems introduced in Section \ref{sect:problem_Des} (constructing the parametric map, selecting the best two already existing bases for a new parameter configuration from accuracy point of view, and selecting the dimension of the reduced basis) we present in detail the proposed solution approaches and the corresponding numerical results.

To assess the performance of  probabilistic models we employ various cross validation tests. The dimensions of the training and testing data sets are chosen empirically based on the number of samples. For artificial neural network models the number of hidden layers and neurons in each hidden layer vary for each type of problems under study. The squared-exponential-covariance kernel \eqref{eq_cov} is used for Gaussian process models.

\subsection{One-dimensional Burgers' equation}\label{subsec:Burgers}
\ifx
Here we propose two alternative approaches to select the reduced basis size that accounts for specified accuracy levels in the reduced order model solutions.
Assume we have a probability space $(\Omega, \mathcal{F},\mathcal{P})$. These techniques employ construction of probabilistic models via neural network and Gaussian
process, $ \phi: X \rightarrow \hat{y} $ where $\phi$ is the transformation function that learns through the input features $X$ to estimate the deterministic
output $y \in \mathbb{R}$ through a real-valued random variable $\hat{y} : \Omega \to \mathbb{R}$.
\fi

Burgers' equation is an important partial differential equation from fluid mechanics \cite{burgers1948mathematical}. The evolution of the velocity $u$ of a fluid evolves according to
\begin{equation}
\frac{\partial u}{\partial t} + \nu u\frac{\partial u}{\partial x} = \mu \frac{\partial^2 u}{\partial x^2}, \quad x \in [0,L], \quad t \in (0,t_\textnormal{f}],\label{eqn:Burgers-pde}
\end{equation}
with $t_\textnormal{f} = 1$ and $L=1$. Here $\mu$ is the viscosity coefficient.The parameter $\nu$ has no physical significance and it is used to control the non-linear effect of the advection term. The model has homogeneous Dirichlet boundary conditions $u(0,t) = u(L,t) = 0$, $t \in (0,t_\textnormal{f}]$. A smooth initial condition is used, described by a seventh degree polynomial and shown in Figure \ref{Fig::1D-Burgers-IC}.


The discretization uses a spatial mesh of $N_s=201$ equidistant points on $[0,L]$, with $\Delta x=L/(N_s-1)$. A uniform temporal mesh with $N_t=301$ points  covers the interval $[0,t_\textnormal{f}]$, with $\Delta t=t_\textnormal{f}/(N_t-1)$. The discrete velocity vector is ${\boldsymbol u}(t_j)\approx [u(x_i,t_j)]_{i=1,2,..,\Ns} \in \mathbb{R}^{\Ns}$, $N=j=1,2,..N_t$, where $\Ns=N_s-2$  (the known boundaries are removed). The semi-discrete version of the model \eqref{eqn:Burgers-pde} is:

\begin{equation}\label{eqn:Burgers-sd}
 {\bf u}'  =  -\nu {\bf u}\odot A_x{\boldsymbol u} + \mu A_{xx}{\boldsymbol u},
\end{equation}
where ${\bf u}'$ is the time derivative of ${\bf u}$, and $A_x,A_{xx}\in \mathbb{R}^{\Ns\times \Ns}$ are the central difference first-order and second-order space derivative operators, respectively, which take into account the boundary conditions too.  The model is implemented in Matlab and the backward Euler method is employed for time discretization. The nonlinear algebraic systems are solved using Newton-Raphson method and the allowed number of Newton iterations per each time step is set to $50$. The solution is considered accurate enough when the euclidian norm of the residual is less then $10^{-10}$.

 The viscosity parameter space is set to the interval $[0.01,1]$. Smaller values of $\mu$ correspond to sharper gradients in the solution, and lead to a dynamics that is more difficult to approximate by a reduced order model.

\begin{figure}[h]
  \centering
\includegraphics[scale=0.37]{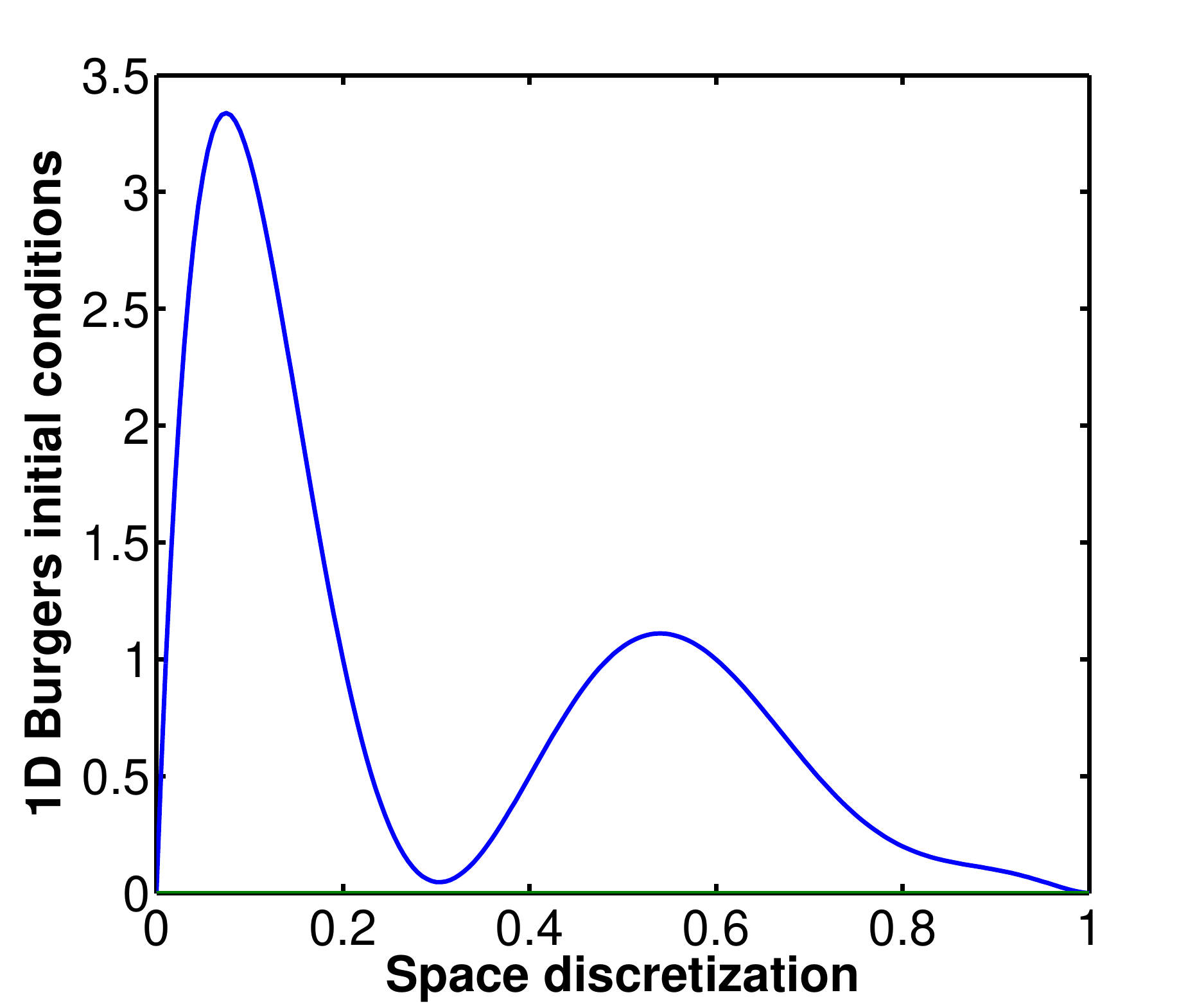}
\caption{Seventh order polynomial used as initial conditions for 1D Burgers model.
\label{Fig::1D-Burgers-IC}}
\end{figure}


\subsection{Designing the parametric map}\label{sec:parametric_map}

We seek to build a parameter map for the 1D-Burgers model for the viscosity domain consisting in the interval $[0.01,1]$. The non-physical parameter $\nu$ is set to $1$. As discussed in Section \ref{sect:problem_Des}, we take the following steps. First we construct probabilistic models to approximate the error of reduced order solution. Next, we identify ``$\mu$-feasible'' intervals $[d_\ell,d_r]$  in the parameter space such that local reduced order model depending only on the high-fidelity trajectory at $\mu$ is accurate within the prescribed threshold for any $\mu^0 \in [d_\ell,d_r]$. Finally, a greedy algorithm generates the parametric map by covering the parameter space with a union of $\mu_i$ feasible intervals.
\begin{equation}
  [0.01,1] \subset \bigcup_{i=1}^M\, \left[d_\ell^i,d_r^i\right],
\end{equation}
where each $\mu_i$-feasible interval is characterized by an error threshold $\bar \varepsilon_i$ (which can vary from one interval to another). This relaxation is suggested since for intervals associated with small parameters $\mu$, it is difficult to achieve small reduced 
order models errors similar to those obtained for larger parametric configurations. One way to mantain the error thresholds constant is to start with a larger proposal.

In existing reduced basis methods a global reduced order model depending on multiple high-fidelity trajectories is constructed. In contrast, our approach decomposes the parameter space into smaller regions where the local reduced order model solutions are accurate within some  tolerance levels. Since the local bases required for the construction of the local reduced order models depend on only a single full simulation, the size of the reduced manifolds is small, leading to lower on-line computational complexity.

\subsubsection{Error estimation of ROM solutions }
\label{sect:err_estimate}


 We first construct probabilistic models
\begin{equation}\label{eqn:prob_model_no_scale}
  \phi: {\bf z} \rightarrow \hat{\varepsilon},
\end{equation}
 where the input features ${\bf z}$ include a new viscosity parameter value $\mu^0$, a parameter value $\mu$ associated with the full model run that generated the basis $U_{\mu}$, and the dimension of the reduced manifold. The target $\hat{\varepsilon}$ is the estimated error of the reduced order model solution at $\mu^0$ using the basis $U_{\mu}$ and the corresponding reduced operators computed using the Frobenius norm

\begin{equation} \label{eqn:param_rang_err}
\varepsilon = \| [{\bf x}(\mu^0,t_i)]_{i=1,..,N_t} - U_{\mu}[{\bf \tilde x}(\mu^0,\mu,t_i)]_{i=1,..,N_t} \|_F.
\end{equation}

The training data set includes equally distributed values of $\mu$ and $\mu^0$ over the entire interval, $\mu \in \{0.1,0.2,\dots,0.9,1\}$ and $\mu^0 \in \{ 0.01, 0.02, \dots, 0.99, 1 \}$, respectively, reduced basis dimensions spanning
 the interval $[4,5,\dots,14,15]$ and the reduced order model error $\varepsilon$. The entire training data set contains nearly $12,000$ samples, and for each sample a high-fidelity model solution is calculated. Figure \ref{fig:parameter_contour}(b) shows isocontours of the reduced order model error $\varepsilon$ for viscosity parameter values $\mu^0$ and various POD basis dimensions. The design of the reduced order models relies on the high-fidelity trajectory for $\mu=0.8$.
 Since target values $\varepsilon$ vary over a wide range (from  300 to $ 10 ^{-6}$) we consider the logarithms of the errors $\log({\varepsilon})$ to decrease the variance of the predicted results, i.e.
 \begin{equation}\label{eqn:prob_model_scale}
  \phi: {\bf z} \rightarrow \widehat{\log(\varepsilon)}.
\end{equation}
 Figure \ref{fig:parameter_contour_log} shows isocontours for the logarithms of the errors  $\log({\varepsilon})$.
\begin{figure}[h]
  \centering
  \subfigure[Isocontours for the logarithms of the errors $\log({\varepsilon})$ ] {\includegraphics[scale=0.3]{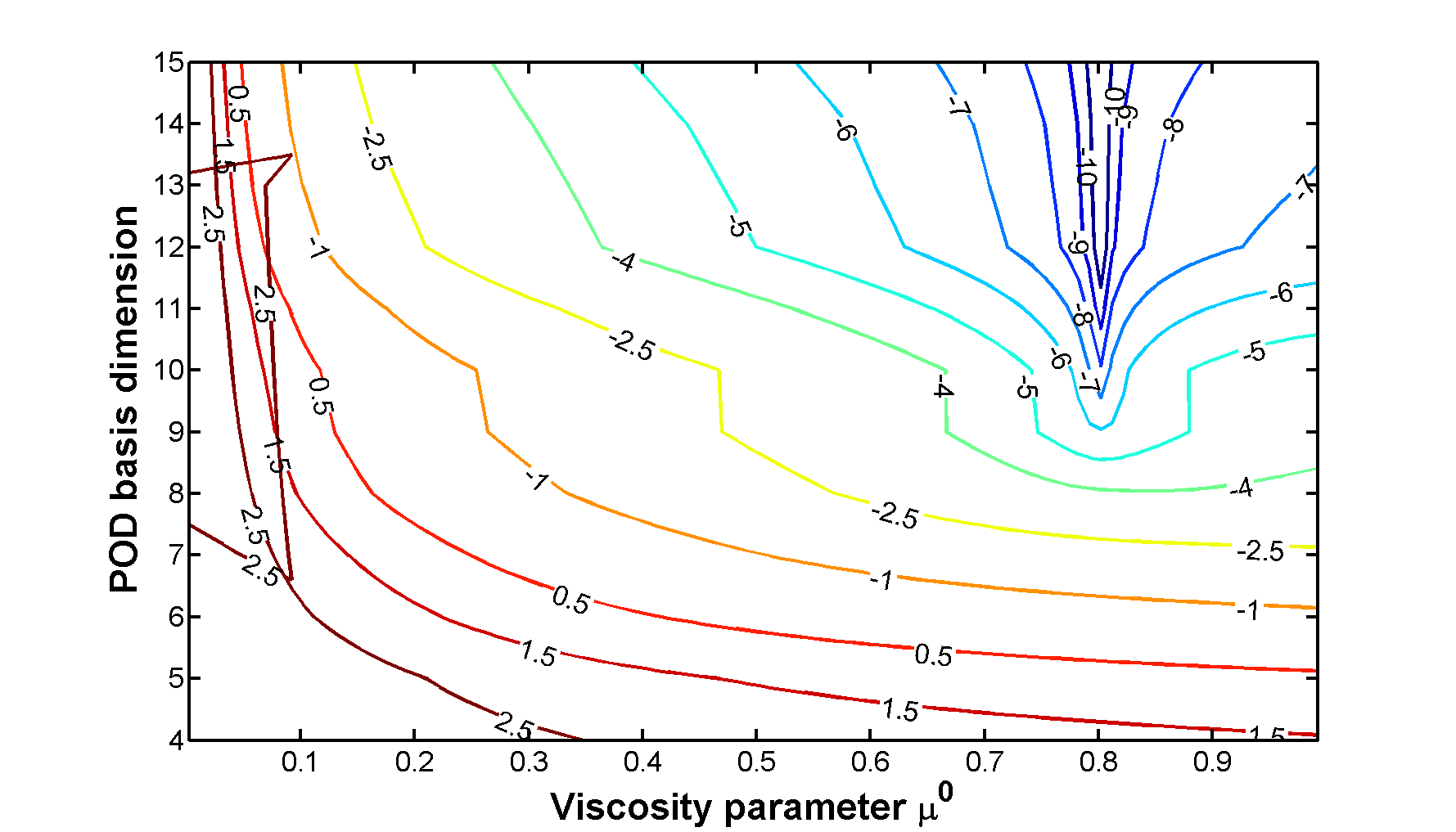}
    \label{fig:parameter_contour_log}}
  \subfigure[Isocontours of the errors ${\varepsilon}$]{\includegraphics[scale=0.3]{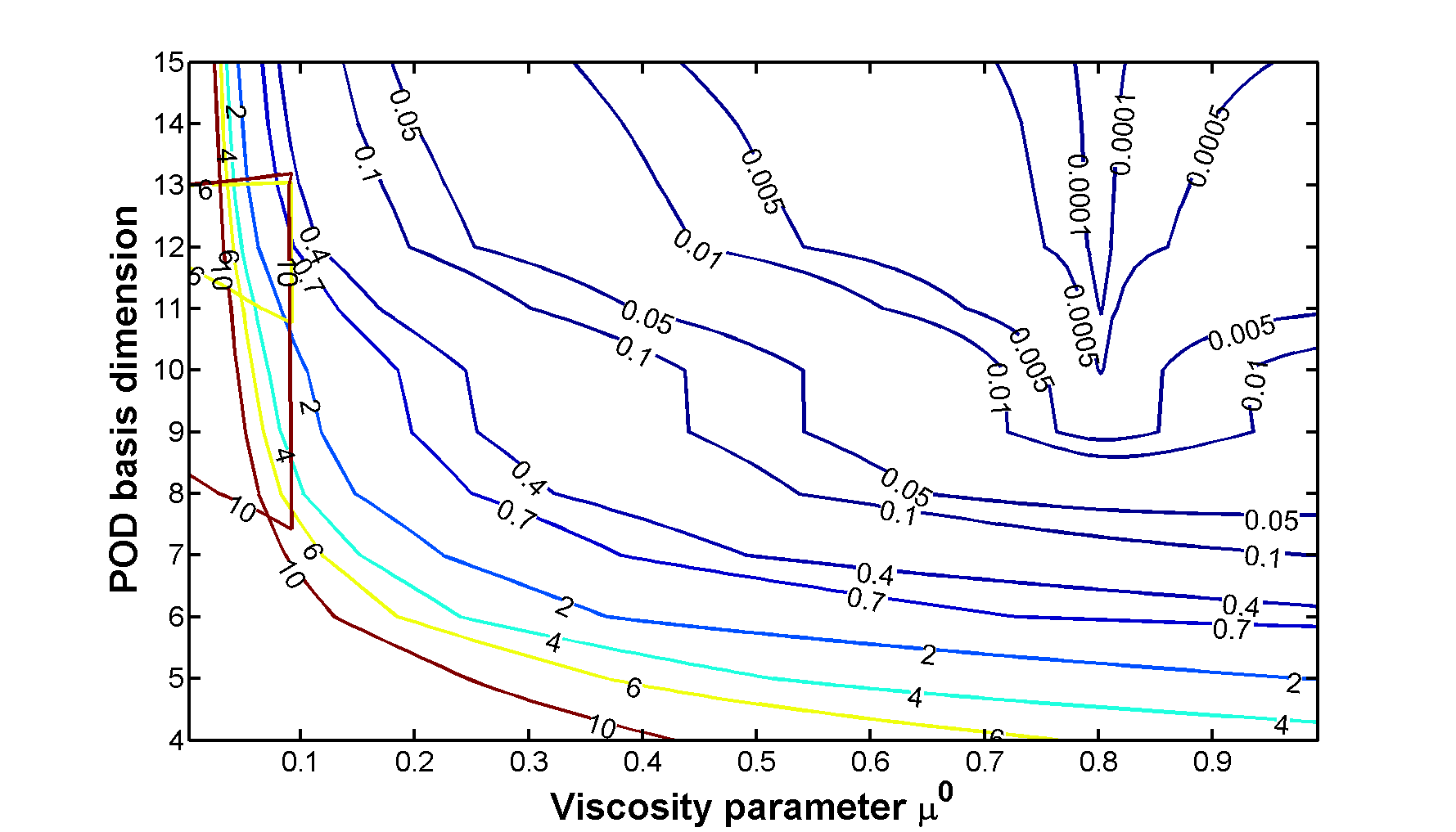}
    \label{fig:parameter_contour_lin}}
\caption{Isocontours of the reduced model errors for different POD basis dimensions and parameters $\mu^0$. The reduced order model uses a basis constructed from the full order simulation with parameter value $\mu=0.8$.}
\label{fig:parameter_contour}
\end{figure}

We construct two probabilistic models for estimating the ROM model errors, the first one uses a Gaussian process with a squared-exponential covariance kernel \eqref{eq_cov} and the second one uses a neural network with six hidden layers and hyperbolic tangent sigmoid activation function in each layer. Tables \ref{tab:Param_log} and \ref{tab:Param_lin} show the averages and variances of errors in prediction provided by GP and ANN for different sample sizes. The results are obtained using a conventional validation with $80\% $ of the whole data set involved in training process and the remaining $20\% $ employed for testing. The misfit is computed using the same formulas presented in \eqref{eqn:err_fold} to evaluate the prediction errors of one-fold set in the K-fold cross validation approach. Table \ref{tab:Param_lin} shows the prediction errors of \eqref{eqn:prob_model_no_scale} computed via equation \eqref{eqn:err_fold} with ${y} = {\varepsilon}$ and $\hat{y} = \hat{\varepsilon}$, i.e. no data scaling; the predictions have a large variance and  a low accuracy. Scaling the data and targeting $\log({\varepsilon})$ results using \eqref{eqn:prob_model_scale}, reduce the variance of the predictions, and increase the accuracy, as shown in Table \ref{tab:Param_log}. The same formula \eqref{eqn:err_fold} with ${y} = {\log(\varepsilon)}$ and $\hat{y} = \widehat{\log(\varepsilon)}$ was applied. In both cases ANN outperforms the GP. Moreover, as the number of data points grows, the accuracy increases and the variance decreases faster for ANN.

%
\begin{table}[H]
\begin{center}
    \begin{tabular}{ | l | l | l |  l | l |}
    \hline
     & \multicolumn{2}{|c|}{GP} & \multicolumn{2}{|c|}{ANN} \\
 \hline
Sample size &  $\textnormal{E}_{\rm fold}$   &  $\textnormal{VAR}_{\rm fold}$    & $\textnormal{E}_{\rm fold}$   &  $\textnormal{VAR}_{\rm fold}$
     \\ \hline
 100 & $ 0.5076 $ & $0.3870$ & $ 0.7237$ & $ 0.9974 $
     \\ \hline
 1000 & $0.2352$ & $ 0.0746 $ & $ 0.0650 $ & $ 0.0397 $
 \\ \hline
 3000 & $ 0.1555 $ & $ 0.0517 $ & $ 0.0110 $ & $ 0.0063 $
 \\ \hline
 5000 & $ 0.0810 $ & $ 0.0176 $ & $ 0.0090 $ & $ 0.0006 $
 \\ \hline
     \end{tabular}
\end{center}
 \caption{Average and variance of error in predictions of \eqref{eqn:prob_model_scale} for ANN and GP using logarithms of errors ($\log({\varepsilon})$) in training data for different sample sizes \label{tab:Param_log}}
  
\end{table}
%
%
\begin{table}[H]
\begin{center}
    \begin{tabular}{ | l | l | l |  l | l |}
    \hline
     & \multicolumn{2}{|c|}{GP} & \multicolumn{2}{|c|}{ANN} \\
 \hline
Sample size &  $\textnormal{E}_{\rm fold}$   &  $\textnormal{VAR}_{\rm fold}$    & $\textnormal{E}_{\rm fold}$   &  $\textnormal{VAR}_{\rm fold}$
 \\ \hline
 100 & $ 8.1523 $ & $ 709.16  $ & $ 9.2341 $ & $ 536.71 $
 \\ \hline
 1000 & $7.9558 $ & $ 593.36 $ & $ 4.4000  $ & $ 424.2900  $
 \\ \hline
 3000 & $6.7521 $ & $ 524.98 $ & $ 4.2238 $ & $  17.6300 $
 \\ \hline
 5000 & $ 2.8229 $ & $ 12.63 $ & $ 2.8134 $ & $ 8.9430 $
 \\ \hline
     \end{tabular}
\end{center}
  \caption{Average and variance of error in predictions of \eqref{eqn:prob_model_no_scale} for ANN and GP using raw errors ( ${\varepsilon}$ ) in training data for different sample sizes \label{tab:Param_lin}}
  
\end{table}

 Figures \ref{fig:ParamHist_NN} and \ref{fig:ParamHist_GP} show the corresponding histogram of the predicted models errors \eqref{eqn:prob_model_scale} and \eqref{eqn:prob_model_no_scale}  using $100$ and $1000$ training samples for both ANN and GP models. 
As the number of training samples increase, the uncertainty in the prediction decreases. The histograms can also asses the validity of GP assumptions \eqref{GP_prior}, \eqref{GP_Dist}, \eqref{GP_training}. The difference between the true and estimated values should behave as samples from the distribution $ \mathcal{N} (0, \sigma_n^2) $ \cite{drohmann2015romes}. In our case they are hardly normally distributed and this indicates that the data set for the problems we are working with, are not from Gaussian distributions.
%
\begin{figure}[h]
  \centering
  \subfigure[Prediction errors $\log(\varepsilon) - \widehat{\log(\varepsilon)}$ - 100 samples] {\includegraphics[scale=0.4]{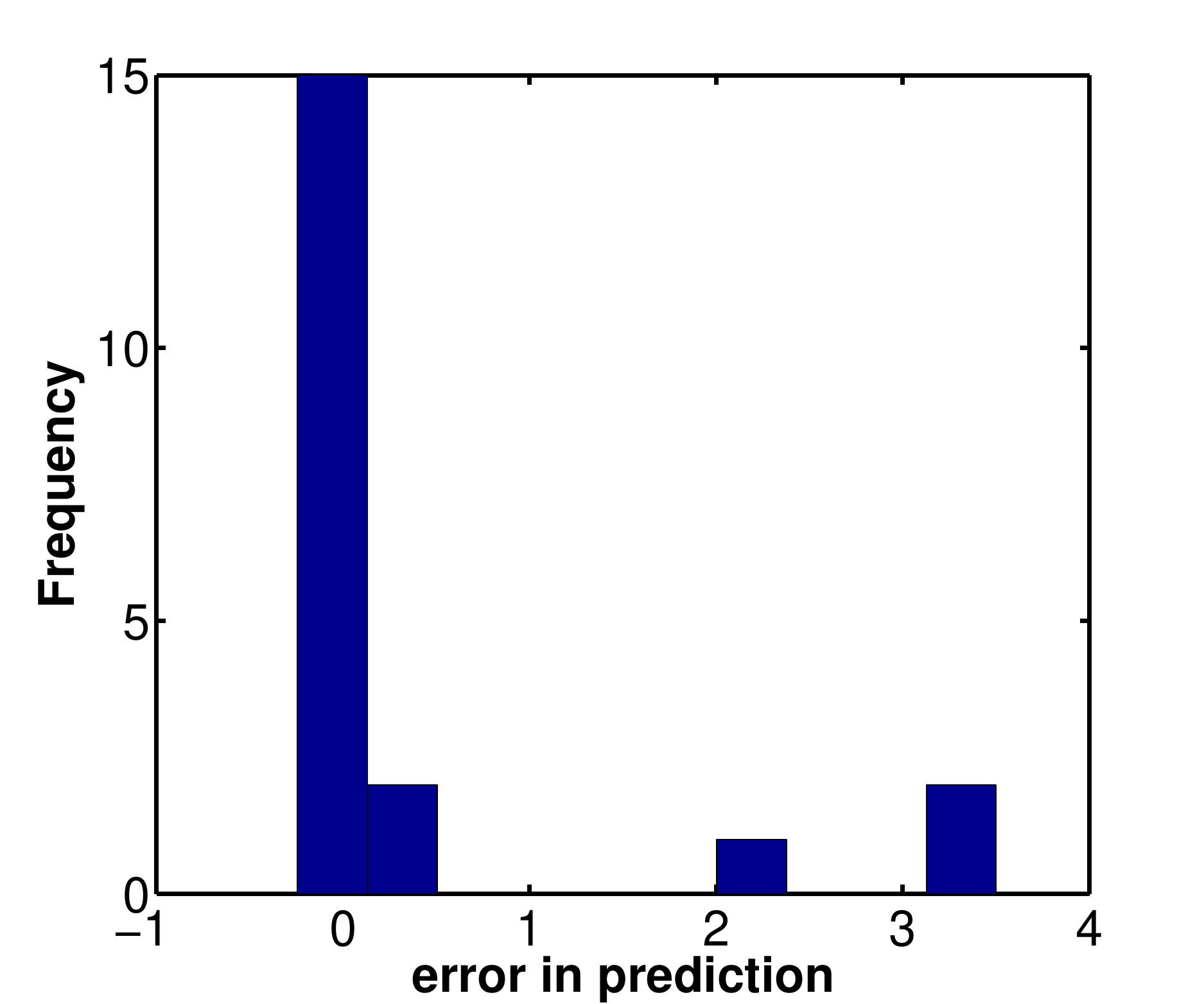}}
  \subfigure[Prediction errors $\varepsilon - \hat{\varepsilon}$ - 100 samples]{\includegraphics[scale=0.4]{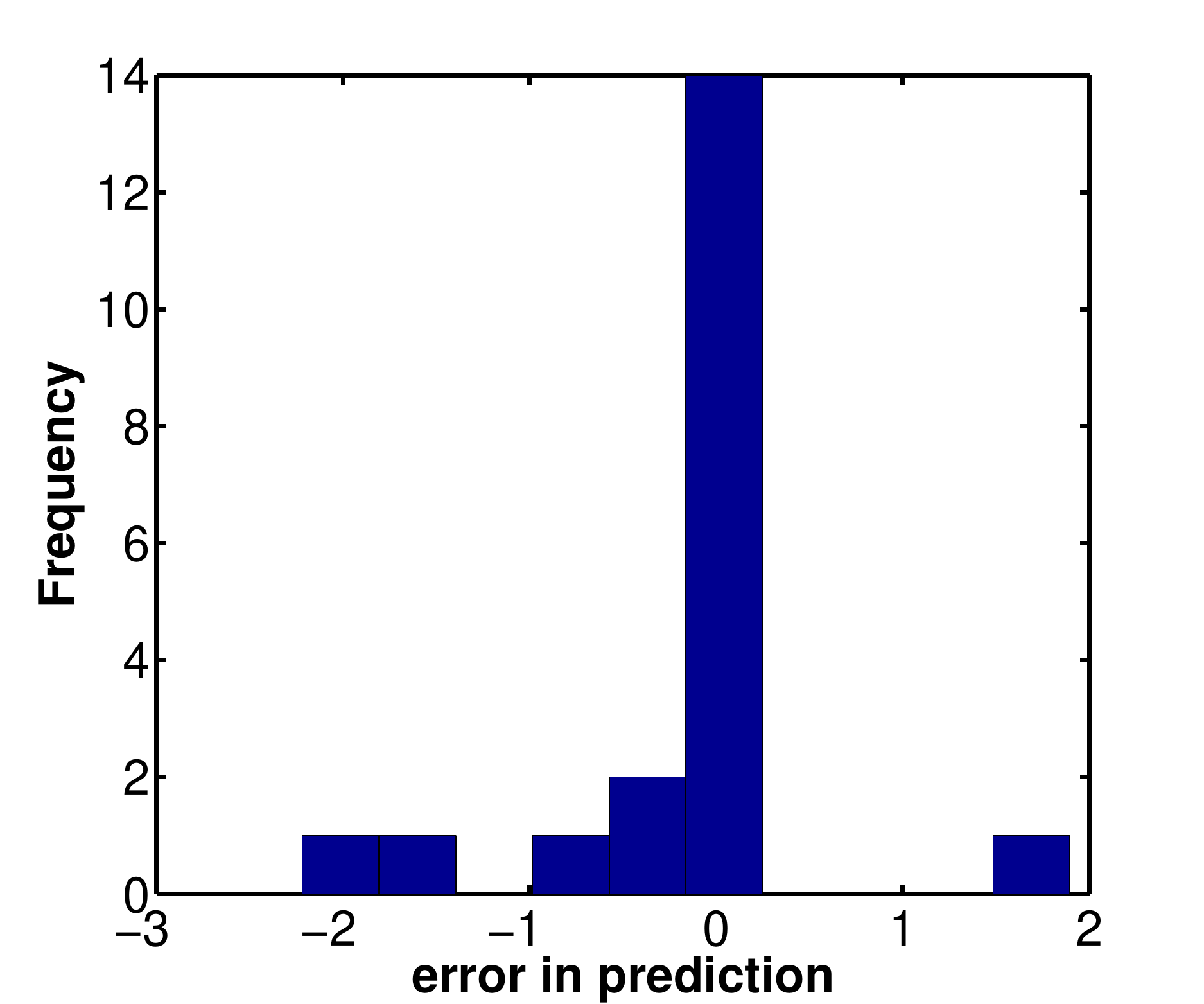}}
  \subfigure[Prediction errors $\log(\varepsilon) - \widehat{\log(\varepsilon)}$ - 1000 samples] {\includegraphics[scale=0.4]{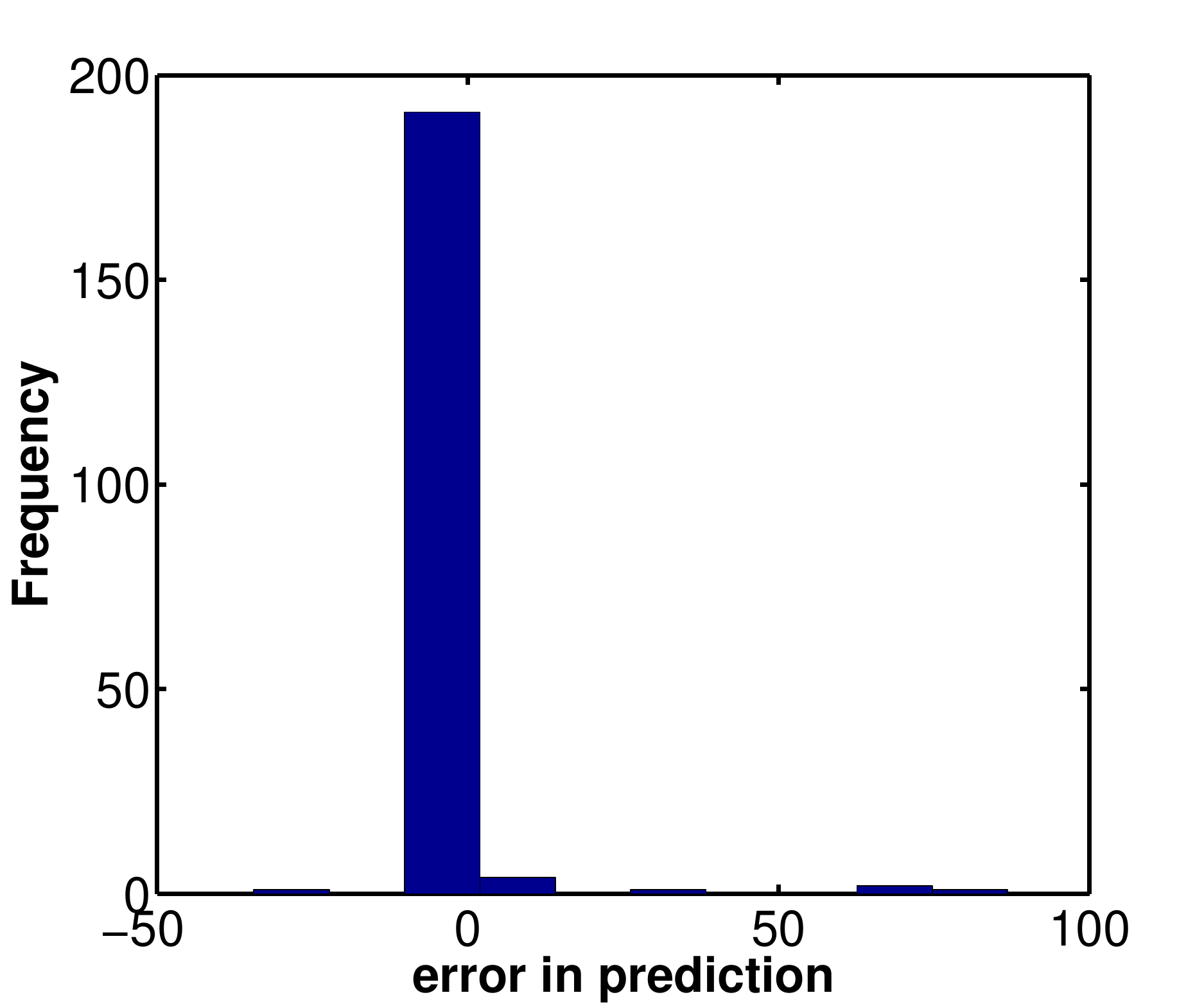}}
  \subfigure[Prediction errors $\varepsilon - \hat{\varepsilon}$ - 1000 samples]{\includegraphics[scale=0.4]{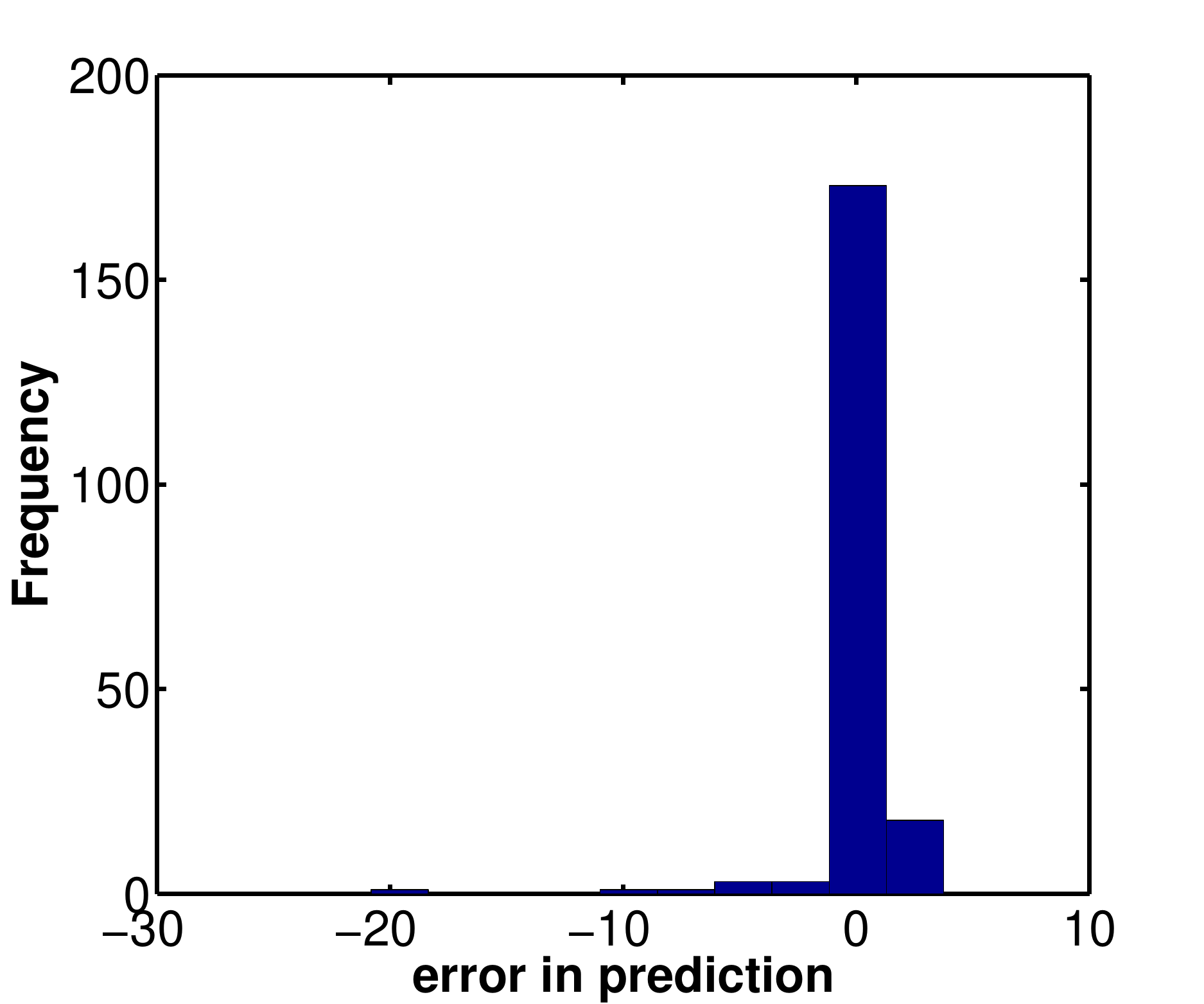}}
\caption{Histogram of errors in prediction using ANN.
\label{fig:ParamHist_NN}}
\end{figure}
%
%

\begin{figure}[h]
  \centering
  \subfigure[Prediction errors $\log(\varepsilon) - \widehat{\log(\varepsilon)}$ - 100 samples] {\includegraphics[scale=0.4]{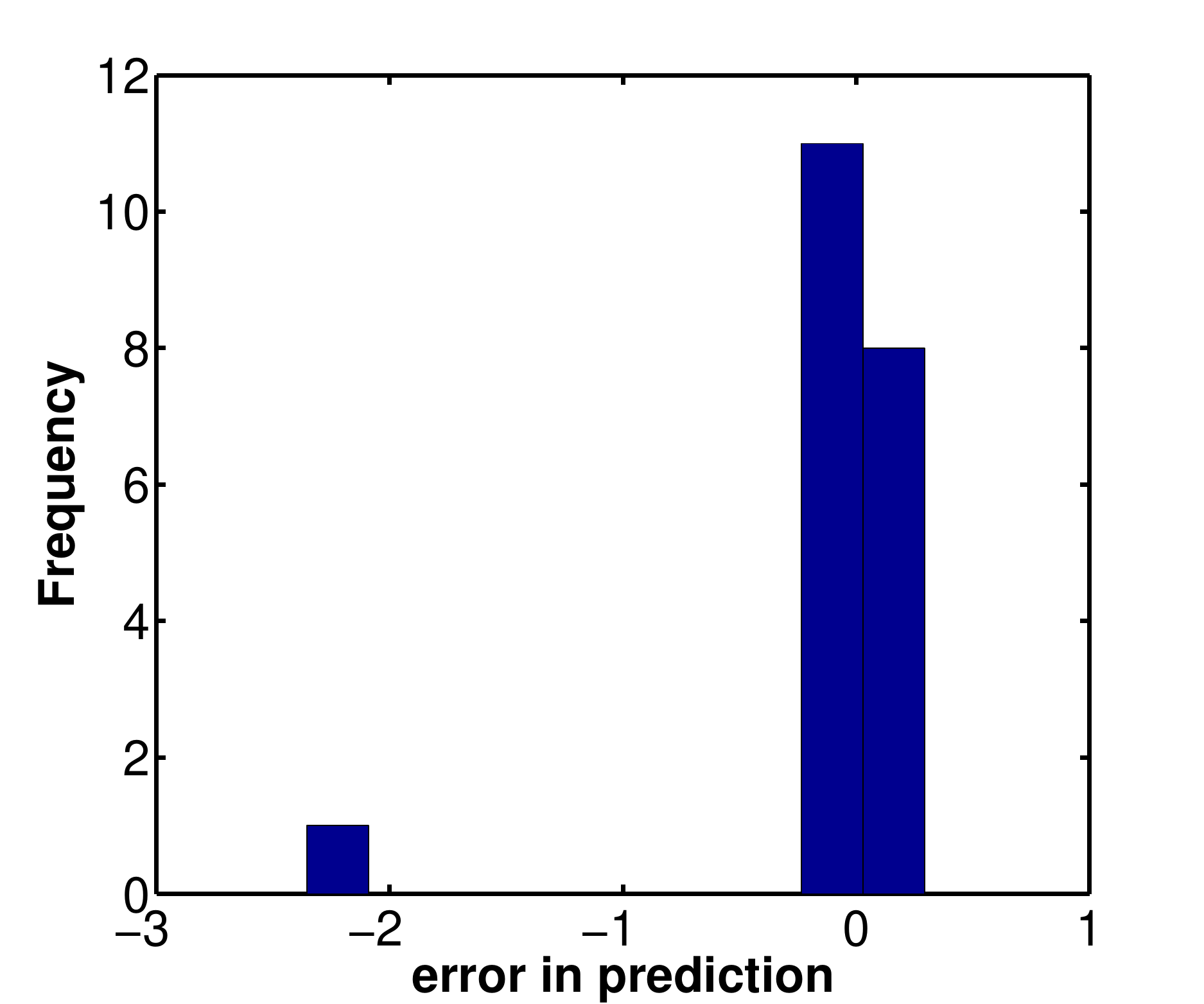}}
  \subfigure[Prediction errors $\varepsilon - \hat{\varepsilon}$ - 100 samples]{\includegraphics[scale=0.4]{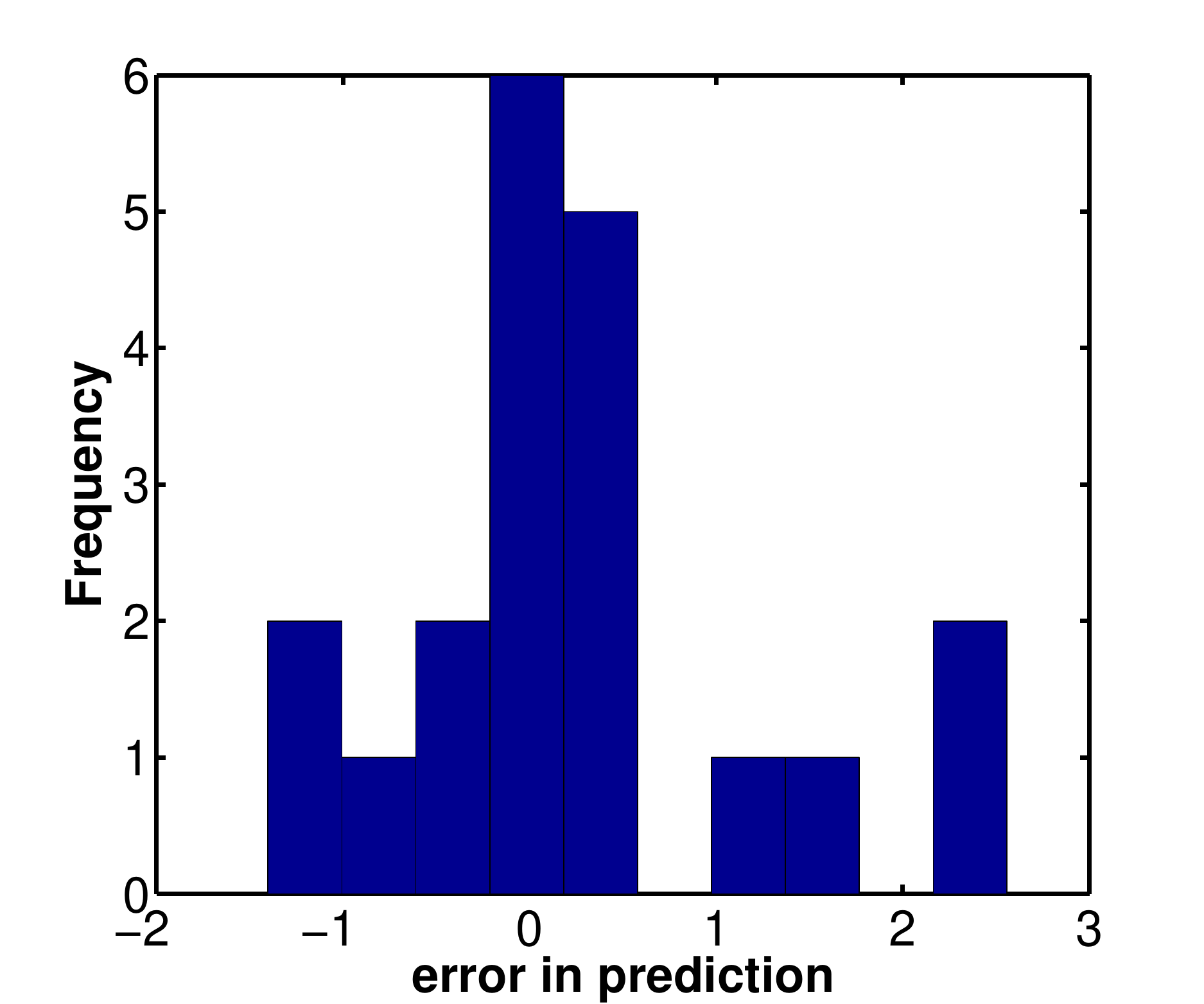}}
  \subfigure[Prediction errors $\log(\varepsilon) - \widehat{\log(\varepsilon)}$ - 1000 samples]{\includegraphics[scale=0.4]{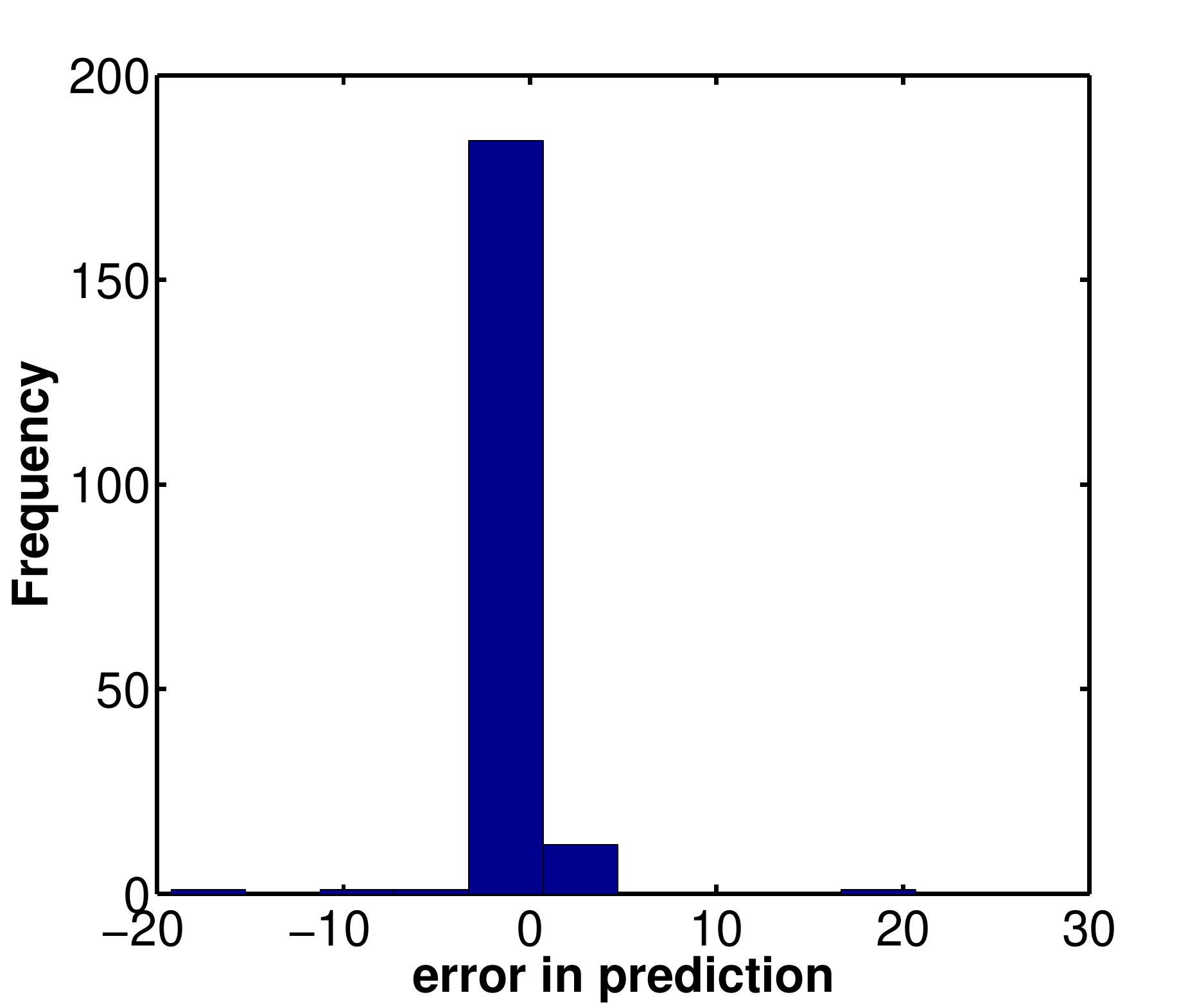}}
  \subfigure[Prediction errors $\varepsilon - \hat{\varepsilon}$ - 1000 samples]{\includegraphics[scale=0.4]{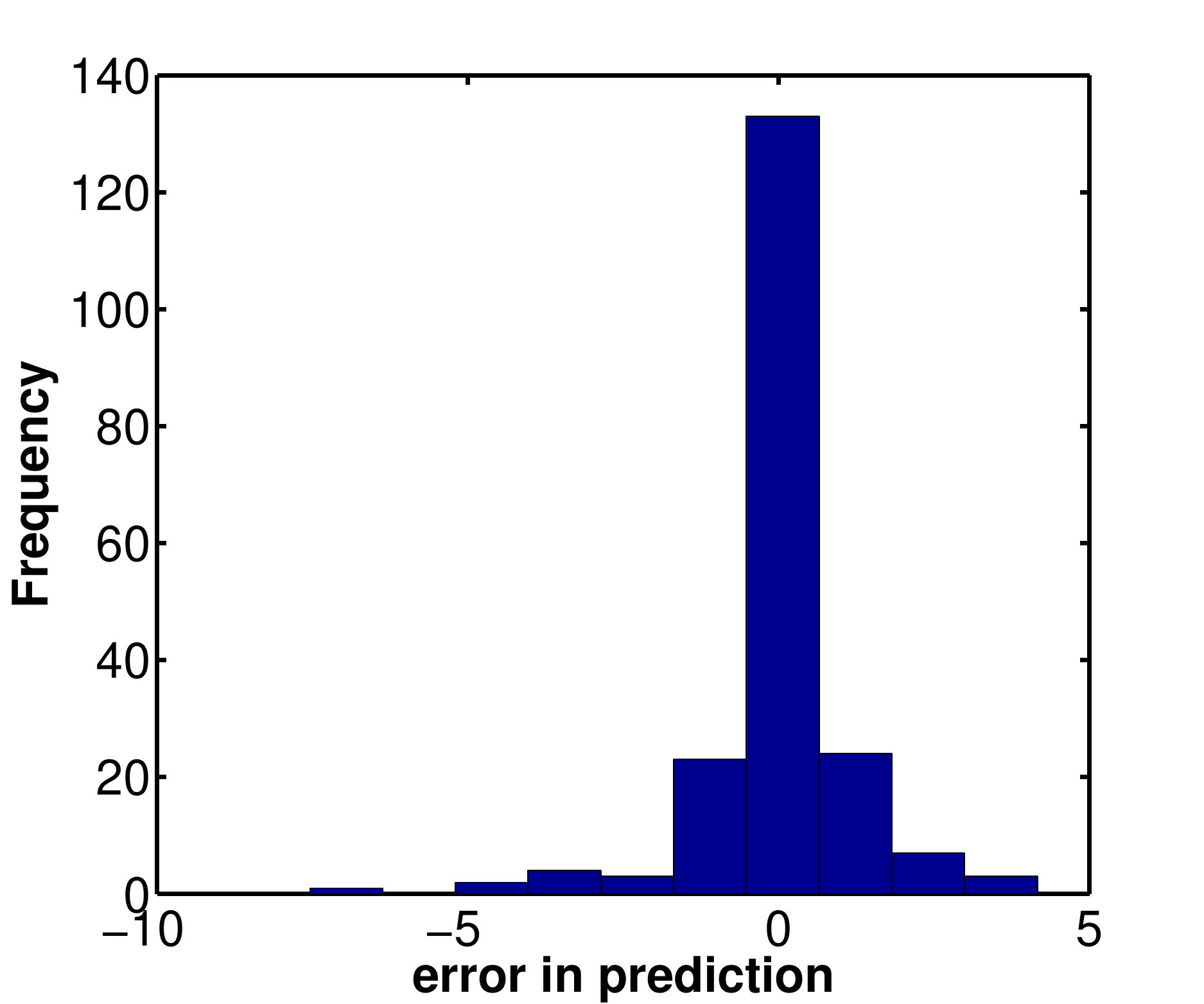}}
\caption{Histogram of errors in prediction using GP.
\label{fig:ParamHist_GP}}
\end{figure}

 Scaling the data and targeting $\log{\varepsilon}$ errors clearly improve the performance of our probabilistic models. Consequently for the rest of the manuscript we will only use model \eqref{eqn:prob_model_scale}. To asses the quality of the probabilistic models a five-fold cross-validation process is also used. The results computed using formula \eqref{eqn:err_fold_average} are shown in Table \ref{tab:experm2}.  Neural network outperforms the Gaussian process and estimates the errors more accurately. It also has less variance than the Gaussian process which indicates it has more stable predictions.

%
\begin{table}[H]
\begin{center}
    \begin{tabular}{ | l | l | l |}
    \hline
   & $\textnormal{E} $  & $ \textnormal{VAR} $
      \\ \hline
 ANN & $0.004004$ & $2.16 \times 10^{-6	}$
     \\ \hline
 GP & $0.092352$ & $ 1.32 \times 10^{-5} $
 \\ \hline
     \end{tabular}
\end{center}
 \caption{ANN and GP statistical results  over 5 fold cross validation.}
  \label{tab:experm2}
\end{table}

Figure \ref{fig:expm2_error_estimates} illustrates the average of errors in predictions over five different ANN and GP configurations. In each configuration, the machine is trained on random $80 \%$ split of data set and tested on the fixed selected test data shown in figure \ref{fig:expm2_error_estimates}. Getting the average of predictions on different trained models decreases the bias in predictions.

%
\begin{figure}[H]
	\begin{centering}
	\includegraphics[width=0.55\textwidth, height=0.45\textwidth]{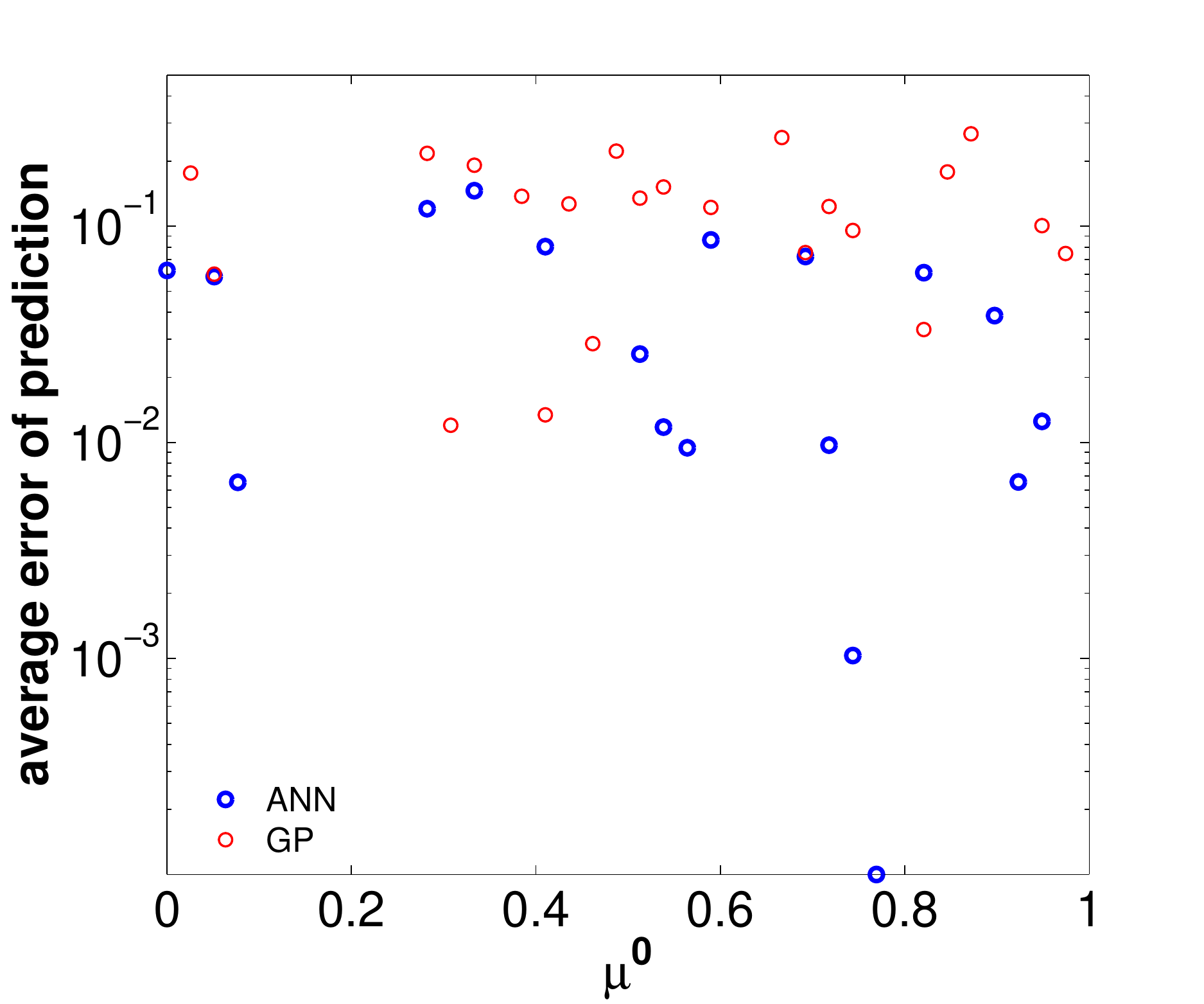}
    \caption{The average of errors in predictions using five different trained models}
	\label{fig:expm2_error_estimates}
	\end{centering}
\end{figure}
%


\subsubsection{Construction of a $\mu-$feasible interval} \label{sec:feasible_interval}
 We saw in the previous subsection that probabilistic models can accurately estimate the error $\varepsilon$ \eqref{eqn:param_rang_err} associated with reduced order models. Thus we can employ them to establish a range of viscosity parameters around $\mu$ such that the reduced order solutions depending on $U_{\mu}$ satisfy some desired accuracy level. More precisely, starting from parameter $\mu$, a fixed POD basis dimension and a tolerance error $\log(\bar{\varepsilon})$, we are searching for an interval $[d_l, d_r]$ such that the estimated prediction $\widehat{\log(\varepsilon)}$ of the true error $\log(\varepsilon)$ \eqref{eqn:param_rang_err} meets the requirement
\begin{equation}\label{eqn:inequality_constraint}
 \widehat{\log(\varepsilon)}<\log(\bar{\varepsilon}), \forall \mu^0 \in [d_l, d_r].
\end{equation}

Our proposed strategy makes use of a simply incremental approach by sampling the vicinity of $\mu$ to account for the estimated errors $\widehat{\log(\varepsilon)}$  forecasted by the probabilistic models defined before. A grid of new parameters $\mu^0$ is build around $\mu$ and the machines predict the errors outward of $\mu$. Once the machines outputs are larger than the prescribed error $\log(\bar{\varepsilon}),$ the previous $\mu^0$ satisfying the constrain \eqref{eqn:inequality_constraint} is set as $d_l$, for $\mu^0 < \mu$ or  $d_r$ for $\mu^0 > \mu$.

Figure \ref{fig:expm2_range} illustrates the range of parameters estimated by the neural network and Gaussian process against the true feasible interval and the results show good agreement. For this experiment we set ${\mu}=0.7$, dimension of POD=9 and error threshold $\bar{\varepsilon} = 10^{-2}$.  Values of $\mu^0 = \mu \pm 0.001\cdot i,$ $i=1,2,..$ are passed to the probabilistic models. The average range of parameters obtained over five different configurations with neural network is $[0.650, 0.780]$ while in the case of Gaussian process we obtained $[0.655,0.780]$. 
In each configuration, we train the model with $80 \%$ random split of the data set and test it over the fixed test set of figure \ref{fig:expm2_range}.
 For this design, the true range of parameters is $[0.650,0.785]$ underlying the predicting potential of machine learning models.

%
\begin{figure}[H]
	\begin{centering}
	\includegraphics[width=0.65\textwidth, height=0.45\textwidth]{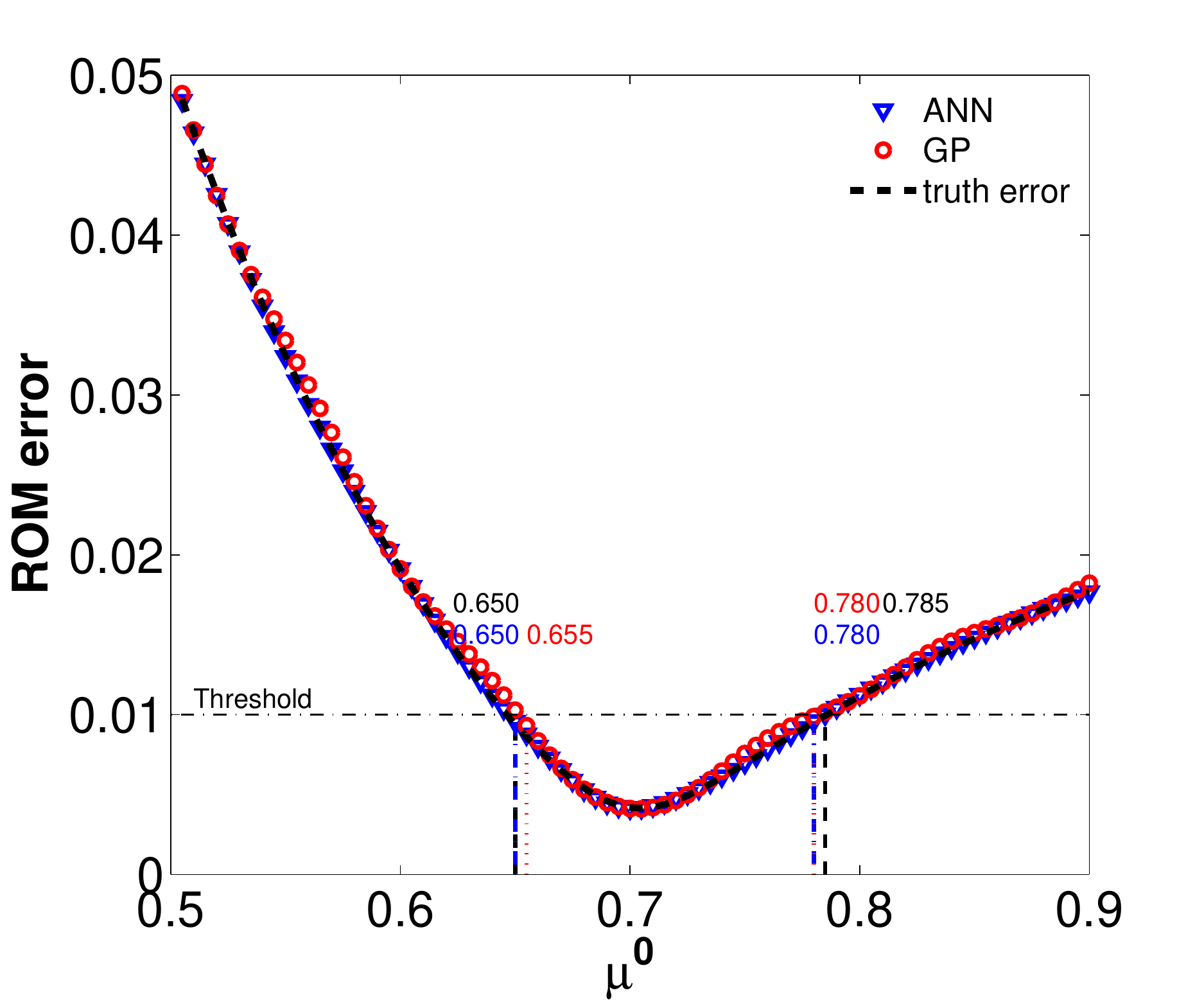}
   \caption{The average range of parameter $\mu^0$ obtained with ANN and GP for dimension of POD=9 and ${\mu}=0.7$. The desired accuracy is $\bar{\varepsilon} = 10^{-2}$. The numbers represent the left and the right edges of the predicted vs the true feasible intervals. }. \label{fig:expm2_range}	
	\end{centering}
\end{figure}
%

\subsubsection{The parametric map as a reunion of $\mu-$feasible intervals}

A reunion of different $\mu_k$-feasible intervals can be designed to cover a general entire 1D-parametric domain $[A,B]$. We refer to this reunion as a parametric map and once such construction is available will allow for reduced order simulations with a-priori error quantification for any value of viscosity parameter $\mu^0 \in [A,B]$.

A greedy strategy is described in Algorithm \ref{alg:map_generation} and its output is a collection of feasible intervals $\cup_{k=1}^n[d_l^k,d_r^k] \supset [A,B]$. Each interval $[d_l^k,d_r^k]$ is associated with some accuracy threshold $\bar{\varepsilon}_k$.  For small viscous parametric values we found out that designing $\mu_k-$feasible intervals associated with higher precision levels (i.e. very small thresholds $\bar{\varepsilon}_k$) is impossible since the dynamics of parametric 1D-Burgers model solutions changes dramatically with smaller viscosity parameters. In consequence we decided to let $\bar{\varepsilon}_k$ vary along the parametric domain to accommodate the solution physical behaviour. Thus a small threshold $\bar{\varepsilon}_0$ will be initially set and as we will advance charting the parameter interval $[A,B]$ from right to left, the threshold $\bar{\varepsilon}_k$ will be increased.

The algorithm starts by selecting the first centered parameter $\mu_0$ responsible for basis generation. It can be set to $\mu_0 = B$ but may take any value in the proximity of $B$,~$\mu_0\leq B$. This choice depends on the variability of parametric solutions in this domain region and by electing $\mu_0$ to differ from the right edge of the domain, the number $n$ of feasible intervals will be decreased.

The next step is to set the threshold $\bar{\varepsilon}_0$ along with the maximum permitted size of the initial feasible interval to be constructed. This is set to $2\cdot r_0$, thus $r_0$ can be referred as the interval radius. Along with the radius, the parameter $\Delta s$ will decide the maximum number of probabilistic model calls employed for the construction of the $\mu_0$-feasible interval. While the radius is allowed to vary during the algorithm iterations, $\Delta s$ is kept constant. Finally the dimension of POD basis has to be selected together with three parameters $\beta_1,~\beta_2$ and $\beta_3$ responsible for changes in the threshold, radius and selecting a new parameter location $\mu_k$ encountered during the procedure.

The instructions between lines $5$ and $20$ generate the $\mu_k$-feasible interval, for the case when the current centered parameter $\mu_k$ represents an interior point of $[d_l^k, d_r^k]$. For situation when $\mu_k = d_l^k$ or $\mu_k = d_l^k$, the threshold has to be increased (by setting $\bar{\varepsilon}_{k+1} = \beta_1\bar{\varepsilon}_k$ at line $22$), since the reduced order model solutions can not satisfy the desired precision according to the estimated probabilistic errors. At this stage the radius is decreased too since we reached a parameter region where model dynamics changes rapidly and a larger feasible interval is not possible. Once the new centered parameter $\mu_{k+1}$ is proposed the algorithm checks if the following constrain is satisfied


\begin{equation}\label{eq::constrain2}
 [d_l^{k+1},d_r^{k+1}] \bigcap \bigg( \bigcup_{i=1}^{k} [d_l^{i},d_r^{i}] \bigg) \neq \emptyset.
\end{equation}

This is achieved by checking the estimated reduced order model solution error at $d_l^k$ using the basis defined by the high-fidelity trajectory at $\mu_{k+1}$ (see instruction $25$). If the predicted error is smaller than the current threshold, assuming a monotonically increasing error with larger distances $d(\mu^0,\mu_{k+1})$, the reduced order model solutions should satisfy the accuracy threshold for all $\mu^0 \in [\mu_{k+1},d_l^k].$ In consequence the equation \eqref{eq::constrain2} will be satisfied for the current $\mu_{k+1}$, if we set $r_{k+1}=\mu_{k+1}-d_l^k$ (see instruction $28$). In the case the error estimate is larger than the present threshold, the centered parameter $\mu_{k+1}$ is updated to the middle point between old $\mu_{k+1}$ and $d_l^k$. For the situation where the monotonic property of the error does not hold in practice, a simply safety net is used at instruction $12$. The entire algorithm stops when $\mu_{k} \leq A.$

For our experiments we set $A=0.01$, $B=1$, $\bar{\varepsilon}_0 = 1.e-2,~\Delta s = 5.e-3,~r_0=0.5,~dim = 9,~\beta_1 = 1.2,~\beta_2 = 0.9$ and $\beta_3=1.4$.  We initiate the algorithm by setting $\mu_0=0.87$, and the first feasible interval $[ 0.7700,1.0500]$ is obtained. Next the algorithm selects $\mu_1=0.73$ with the associated range of $[ 0.6700,0.8250]$ using the same initial threshold level. As we cover the parametric domain from right to left, i.e. selecting smaller and smaller parameters $\mu_k$, the algorithm enlarges the current threshold $\bar{\varepsilon}_k$, otherwise the reduced order models would not satisfy the initial precision. We continue this process until we get the threshold $6.25$ with $\mu_{32}=0.021$ and the corresponding feasible interval $[0.00940,0.039]$. The obtained parametric map is depicted in Figure \ref{fig:expm2_numMus} where the associated threshold varies with the parameter change.

\begin{algorithm}
 \begin{algorithmic}[1]
\State Select $\mu_0$ as the right edge of the parameter interval, i.e. $\mu_0 = B$.
\State Set error threshold $\hat{\varepsilon}_0$, step size $\Delta s$ for selection of new parameter locations $\mu^0$, the maximum search radius $r_0$, dimension of POD basis $dim$ and $\beta_1,~\beta_2$ and $\beta_3$.
\State Set $k=0$.
\State DO
 \State FOR i=1 to $int(\frac{r_k}{\Delta s})$
  \State Set $\mu_+^o = \mu_k + i \Delta s$
  \State IF $\phi(\mu_+^o,\mu_k,dim) > \log(\bar{\varepsilon}_k)$  THEN
  \State Set $d_r^k = \mu_k + (i-1) \Delta s$. EXIT.
 \State END IF
 \State END FOR
 \State IF $k>0$ THEN
    \State IF $d_r^k<d_l^{k-1}$ THEN $\mu_k = \frac{\mu_k+d_l^{k-1}}{2}$. GOTO $5$.
    \State END IF
  \State END IF
\State FOR $j=1$ to $int(\frac{r_k}{\Delta s})$
  \State Set $\mu_-^o = \mu_k - j \Delta s$
  \State IF $\phi(\mu_-^o,\mu_k,dim) > \log(\bar{\varepsilon}_k)$  THEN
  \State Set $d_l^k = \mu_k - (j-1) \Delta s$. EXIT.
  \State END IF
   \State END FOR
\State IF (i=1).OR.(j=1) THEN
\State Set $\bar{\varepsilon}_{k} = \beta_1 \cdot \bar{\varepsilon}_k$; $r_{k} = \beta_2 \cdot r_k$;  GOTO $5$.
\State ELSE $\mu_{k+1} = \mu_k - \beta_3 (j-1) \Delta s$; $\bar{\varepsilon}_{k+1} = \bar{\varepsilon}_k$.
\State END IF
\State WHILE $\phi(d_l^k,\mu_{k+1},dim) > \log(\bar{\varepsilon}_{k+1})$ DO
\State $\mu_{k+1} = \frac{\mu_{k+1} + d_l^k}{2}$.
\State END WHILE
\State Set $r_{k+1}=\mu_{k+1}-d_l^k$.
\State $k=k+1$.
\State WHILE $\mu_{k} \geq A$ THEN STOP.

\end{algorithmic}
 \caption{Generation of parametric map for reduced order models usage}
 \label{alg:map_generation}
\end{algorithm}

%
\begin{figure}[H]
	\begin{centering}
	\includegraphics[width=0.65\textwidth, height=0.45\textwidth]{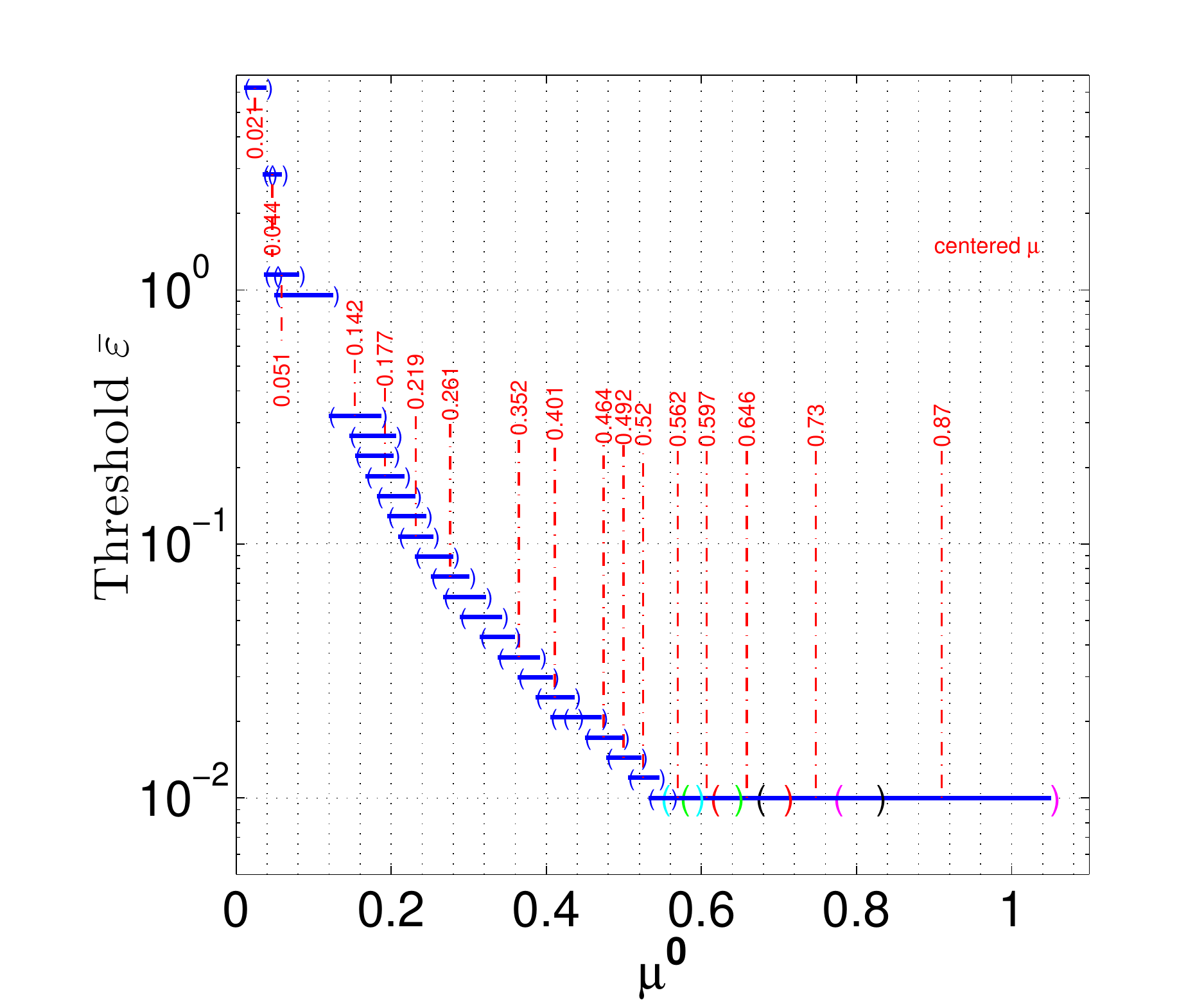}
    \caption{The diffusion parametric map defining the local feasible intervals and their corresponding errors. Associated with one feasible interval there is a centered parameter $\mu$ high-fidelity trajectory that guides the construction of a reduced basis and operators such that the subsequent reduced order model solutions along this interval are accurate within the threshold depicted by the Y-axis. }
	\label{fig:expm2_numMus}
	\end{centering}
\end{figure}

%


\subsection{ Select the best already existing ROMs for a new parameter value $\mu^0$}
\label{subsec:select_best_basis}

An important and practical concern associated with the reduced order models is their lack of robustness with respect to parameter change. Here, we propose to solve another practical problem, i.e. giving a collection of reduced bases computed at various locations in the parameter space, find the one that proposes the most accurate reduced order solution for a new viscosity parameter $\mu^0$. We will rely on similar probabilistic models
built in subsection \ref{sect:err_estimate}. The input features for the GP and ANN are
the new parameter $\mu^0$, a parameter $\mu$ whose corresponding trajectory is used as snapshots for generating basis $U_{\mu}$ and the dimension of the reduced manifold. The target random variable $\hat{y} = \widehat{\log{\varepsilon}}$ is the estimated log of error of the reduced order model solution at $\mu^0$ using the basis $U_{\mu}$.
For our experiments, approximately $12,000$ samples were generated and include different values of POD basis dimensions from $4$ to $15$, new viscosity parameter $\mu^0 \in \left[0.1,1 \right]$ and parameters $\mu_k \in \left[10^{-2},1 \right] $  equally distributed with interval $0.1$ and $0.01$, respectively and the corresponding log of the Frobenius norm of true ROM errors $\log(\varepsilon)$ \eqref{eqn:param_rang_err}.

The data set constructed for this problem is similar to the one employed for designing the errors models. For that problem we fixed the parameter $\mu$ associated with the high-fidelity trajectory used for basis generation and computed the ROM errors corresponding to different $\mu^0$. Here we fixed $\mu^0$ and then computed the ROM errors using different bases $U_\mu$.

The data set is randomly partitioned into $5$ equal size sub-samples and a $5-$ fold cross-validation process is used to asses the quality of the probabilistic models. A neural network with 6 hidden layers and hyperbolic tangent sigmoid activation function in each layer is used while for the Gaussian process we have employed the squared-exponential-covariance kernel \eqref{eq_cov}. Table \ref{tab:experm3} shows numerical results obtained with Neural networks and Gaussian process, over five folds using \eqref{eqn:err_fold_average}. The Gaussian process outperforms the neural network and has less variance which indicates the GP is more accurate and stable for this specific problem than neural network.
\begin{table}[H]
\begin{center}
    \begin{tabular}{ | l | l | l |}
    \hline
 & $\textnormal{E} $  & $ \textnormal{VAR} $
      \\ \hline
 ANN & $0.002839 $ & $1.5496 \times 10^{-5}$
     \\ \hline
 GP & $0.001135 $ &  $ 2.5922 \times 10^{-8} $
     \\ \hline
     \end{tabular}
\end{center}
 \caption{Statistical results of error in predictions in ANN and GP over 5 fold cross validation}
  \label{tab:experm3}
\end{table}

The results in Figure \ref{fig:combin_bases} illustrate the errors in prediction of the reduced order models errors for two viscosity parameters $\mu^0=0.35$ and $0.65$ and various bases represented along the $y$ axis. The mean of the estimates in the case of the ANN are closer than the true errors in comparison with the output of the Gaussian Process for this particular example.


 Moreover we can notice that the estimation curves are crossing initially close to $\mu=0.45$. It suggests that one can choose the high-fidelity trajectory $\mu=0.45$ to construct a reduced order basis such that to obtain similar accuracy levels for both reduced order solutions computed at new viscosity parameters $\mu^0 = 0.35$ and $0.65$. This reveals the non-monotonic property of the reduced order model error with respect to the distance $d(\mu^0,\mu)$.

%
\begin{figure}[h]
  \centering
  \subfigure[Neural Network] {\includegraphics[scale=0.4]{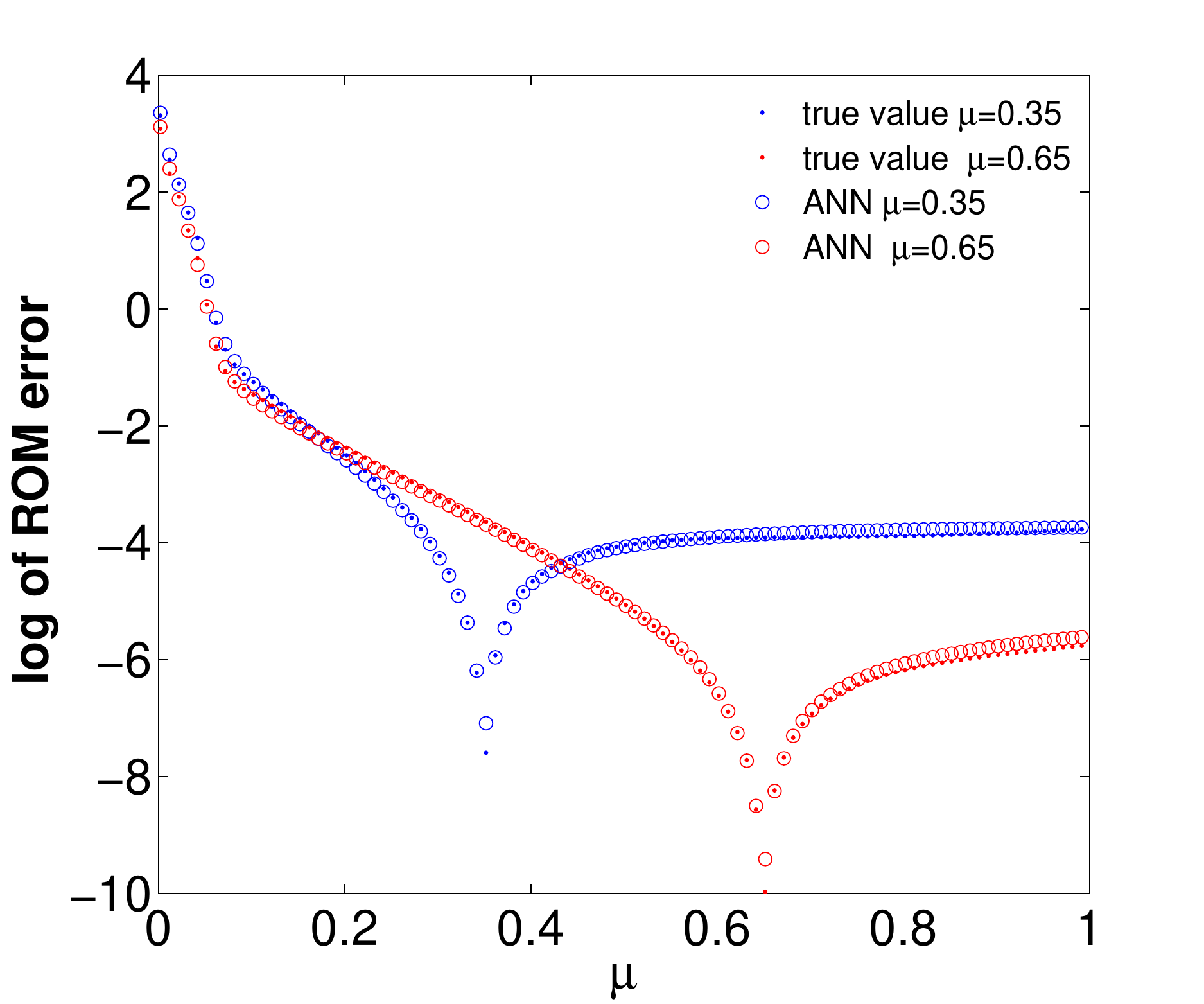}}
  \subfigure[Gaussian Process]{\includegraphics[scale=0.4]{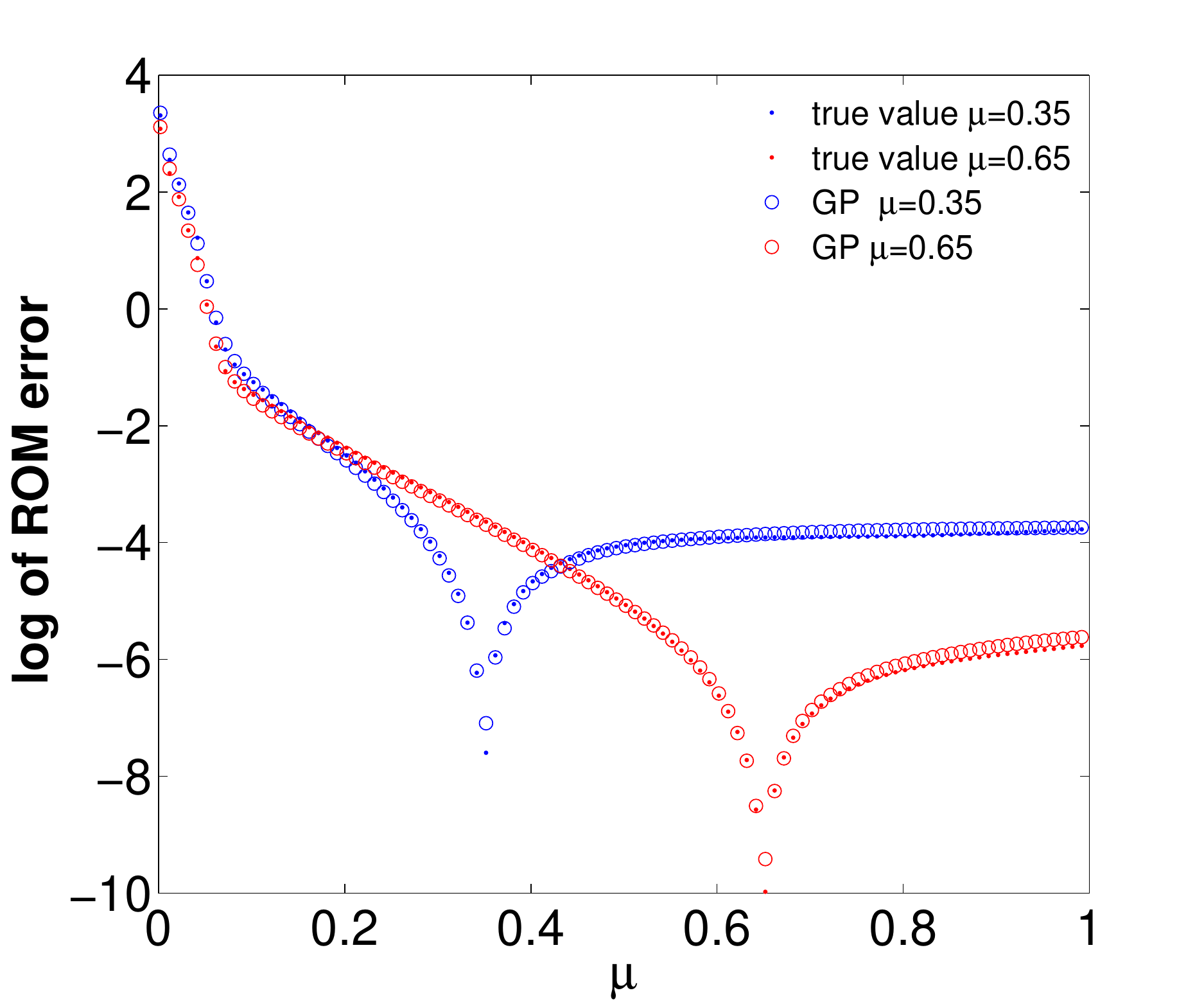}}
\caption {ANN and GP prediction errors ($\log(\varepsilon)-\widehat{\log(\varepsilon)})$ of the ROM errors for dimension of POD basis=12, $\mu^0=0.65$ and $\mu^0=0.35$ and different values of $\mu$.
\label{fig:combin_bases}}
\end{figure}
%
\subsection{Combining Available Information for Accurate ROMs at New Parametric Configurations}
Experiments in subsection \ref{subsec:select_best_basis} revealed the potential of probabilistic models to select a hierarchy of reduced manifolds that produces higher accurate solutions. Figure \ref{fig:combin_bases} depicts the accuracy of reduced order models for $\mu^0=0.35$ and $0.65$ using POD basis computed at various locations in the parameter interval. Assuming only $10$ existing POD subspaces constructed for parameters equally distributed along the interval $[0.1,1]$, $\mu = 0.1,~0.2,~0.3,..$, figure \ref{fig:combin_bases} shows that the most accurate reduced solutions for $\mu^0 = 0.35$ are produced using the bases computed at $\mu_1 = 0.3$ and $\mu_2 = 0.4$. Consequently, the numerical experiments described here focus on the construction of a POD basis for $\mu^0 = 0.35$ by combining the data available for $\mu_1=0.3$ and $\mu_2=0.4$.

The performances of the discussed methods (bases concatenation, Lagrange interpolation of bases in the matrix space and in the tangent space of the Grassmann manifold, Lagrange interpolation of high-fidelity solutions) are shown in the case of three main experiments: variation in the final time $t_f$, in the non-linear advection coefficient $\nu$ and POD basis size. The first two experiments scale the time and space and modify the linear and nonlinear characteristics of the model. For example, in the case of a tiny small final time and advection coefficient, the diffusion linear part represents the main dynamical engine of the model thus it behaves linearly. The results are compared against reduced order models constructed using $U_{\mu_1}$ and $U_{\mu_2}$, respectively.

The experiments make use of a space mesh of $201$ points while $301$ time steps are used. Figure \ref{Fig::Tiny_final_time} illustrates the Frobenius norm error between the high fidelity and reduced order model solutions for the final time $t_f = 0.01$. Panel (a) presents the accuracy results as a function of the advection coefficient $\nu$. Interpolating the high-fidelity solutions leads to the most accurate reduced order model. For large advection coefficients all of the methods suffer accuracy losses. Among the potential explanations we include the constant size of the POD basis and its linear dependence on the viscosity parameter assumed by all of the methods in various forms. Keeping the POD basis size constant represents a source of errors as seen in Figure \ref{fig:parameter_contour} where the viscosity parameter is varied.
\begin{figure}[h]
  \centering
  \subfigure[The nonlinearity model variations  ] {\includegraphics[scale=0.4]{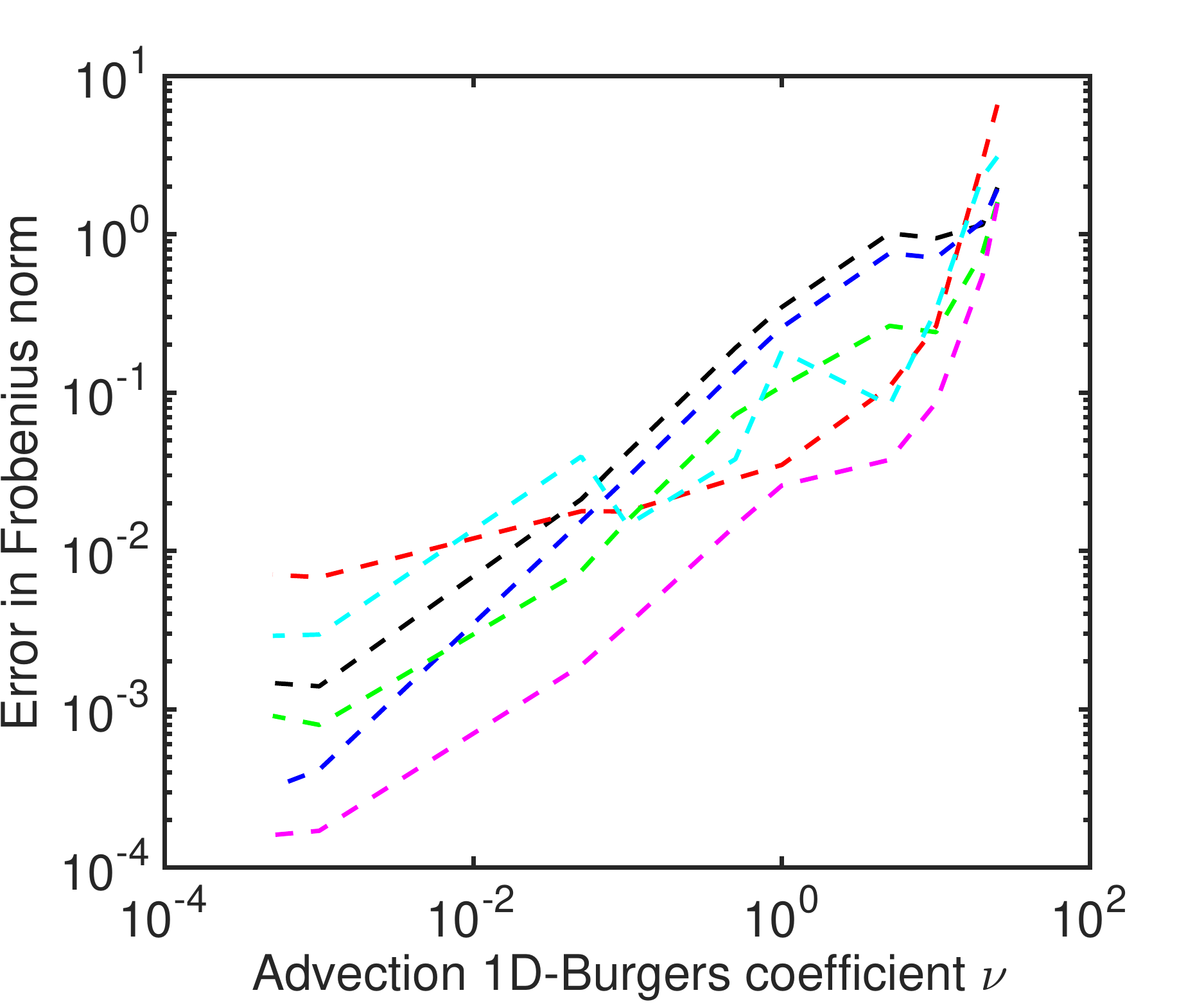}}
  \subfigure[Dimension of POD basis variation]{\includegraphics[scale=0.4]{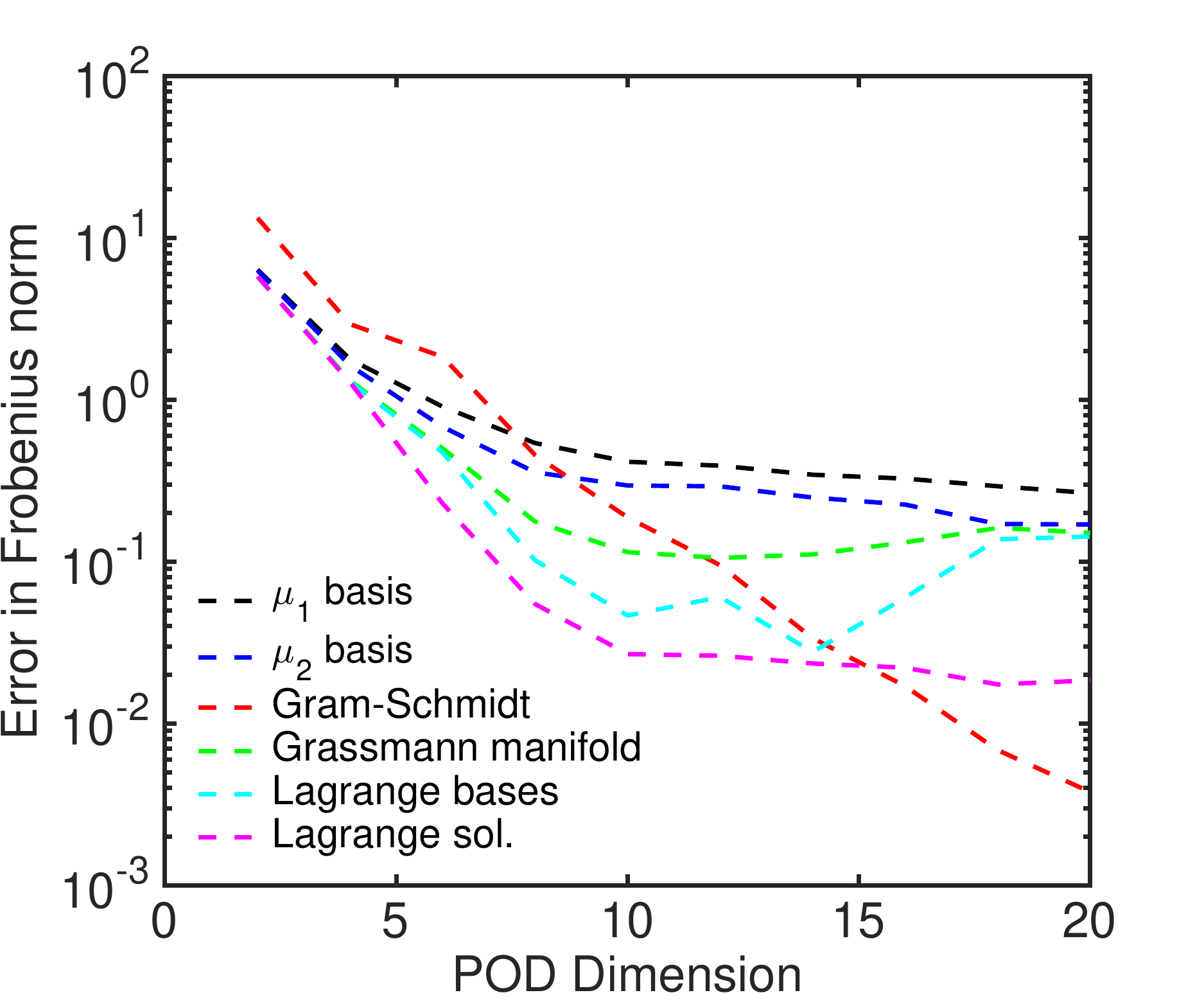}}
\caption{Strategies comparison for generation of accurate ROMs for a new viscosity parameter $\mu^0 = 0.35$ and $t_f = 0.01$.  }
\label{Fig::Tiny_final_time}
\end{figure}

By interpolating the geometry characteristics of the reduced subspaces via subspace angle interpolation or Grassmann manifold approach (only one is shown in the figures since for interpolating two bases the methods are the same) we expect to obtain more accurate reduced order models. While Lagrangian interpolation of the bases is performed in both matrix space and tangent space of the Grassmann manifold (shown in cyan and green in Figure \ref{Fig::Tiny_final_time},  the later approach performs better in this scenario confirming our expectations. The concatenation of bases using Gram-Schmidt was successful only for larger advection coefficients (red curve in Figure \ref{Fig::Tiny_final_time})(a) for a POD size set to $14$.

Increasing the dimension of the basis enhances the so called Gram-Schmidt reduced order model solution accuracy for $\nu=1$ (see Figure \ref{Fig::Tiny_final_time}(b)). For this case Lagrange interpolation in the matrix space shows better performances in comparison with the output of the Grassmann manifold approach.

Next we increase the nonlinearity characteristics of the model by setting the final time to $t_f = 1$ and Figure \ref{Fig::Larger_final_time} illustrates the Frobenius norm errors as a function of the advection coefficient $\nu$ and POD dimension size. The errors produced by the reduced order model derived via Grassmann manifold method are similar with the ones obtained by the surrogate model relying on a POD basis computed via the Lagrange interpolation of the high-fidelity model solutions.

 The Lagrange interpolation of bases in the matrix space is not successful as seen in both panels of figure \ref{Fig::Larger_final_time}. Increasing the POD size to $20$, the Gram-Schmidt approach enhance the accuracy of the solution (see Figure \ref{Fig::Larger_final_time}(b)).
\begin{figure}[h]
  \centering
  \subfigure[The nonlinearity model variations ] {\includegraphics[scale=0.4]{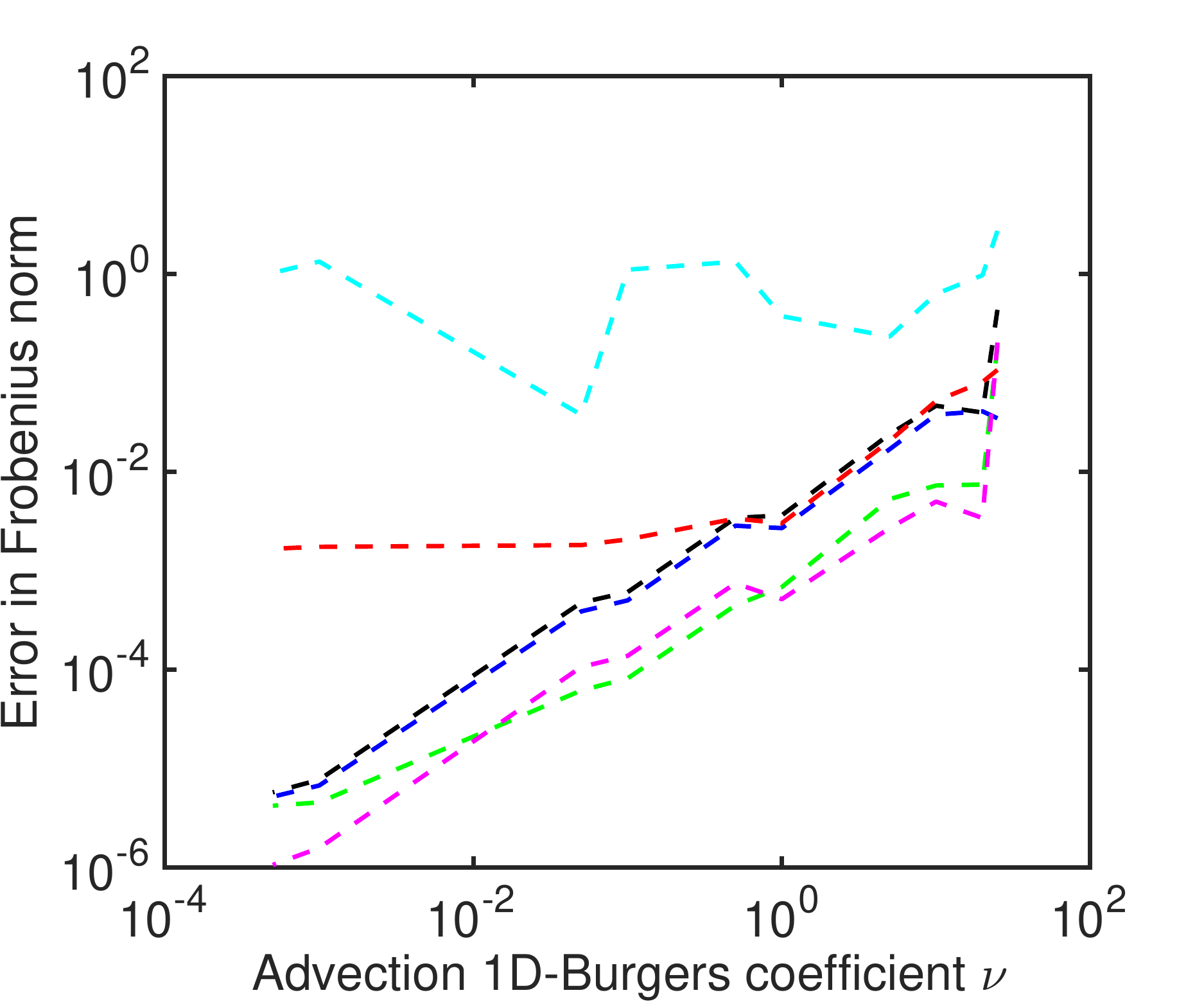}}
  \subfigure[Dimension of POD basis variation]{\includegraphics[scale=0.4]{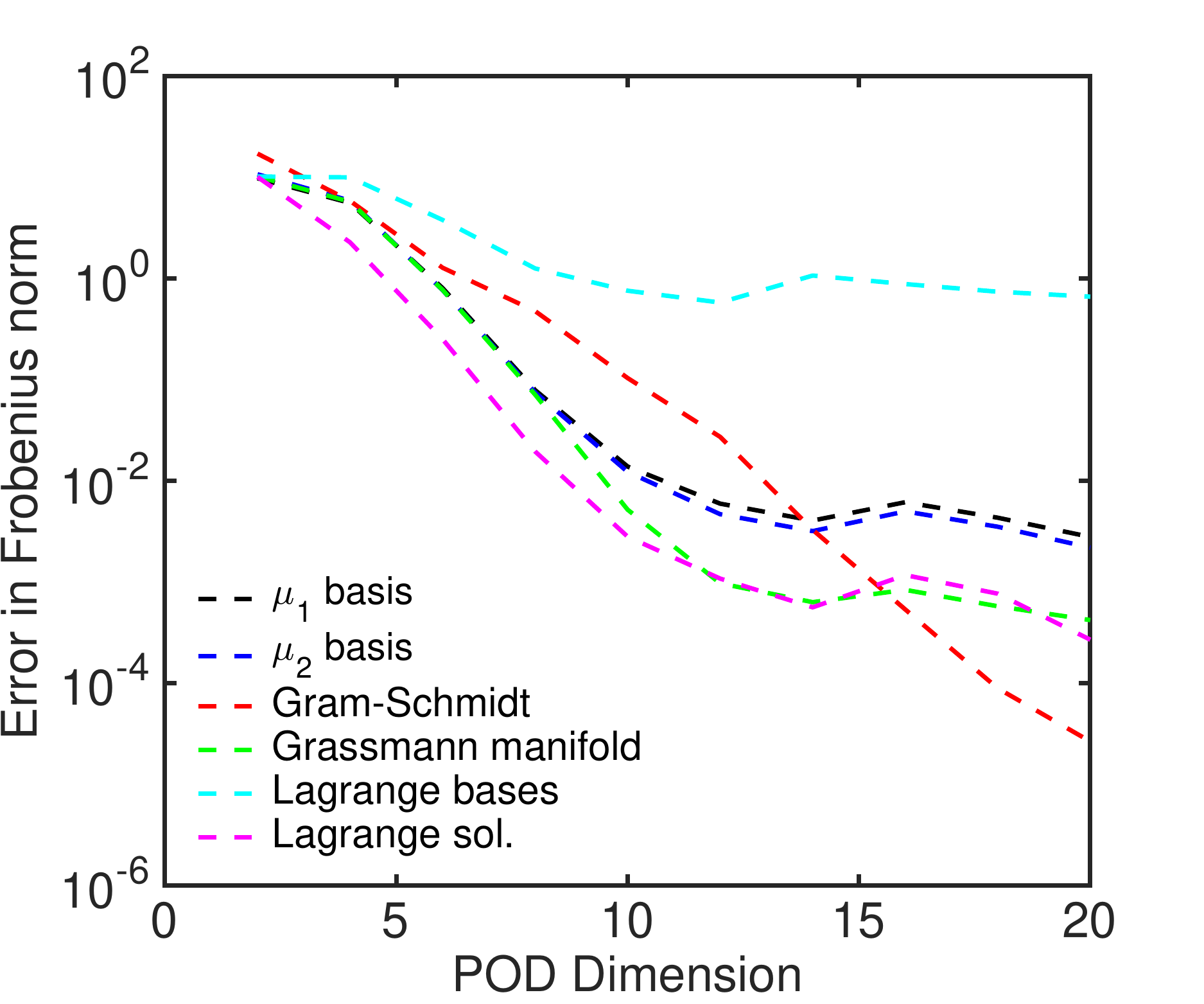}}
\caption{Strategies comparison for generation of accurate ROMs for a new viscosity parameter $\mu^0 = 0.35$ and $t_f = 1$.
\label{Fig::Larger_final_time}}
\end{figure}
%

\subsection{Optimal size of reduced order model}
\label{sect:optimal_base}
 Here we propose two alternative approaches to select the reduced basis size that accounts for specified accuracy levels in the reduced order model solutions. These techniques employ construction of probabilistic models $ \phi: {\bf z} \rightarrow \hat{y} $ via neural network and Gaussian process as stated in section \ref{sect:prob_fram}. The input features ${\bf z}$ for this problem consist of the viscosity parameter $\mu \in [0.01,1]$ and the log of the Frobenius norm of the error between the high-fidelity and reduced order models $log(\varepsilon)$ \eqref{eqn:param_rang_err}. The searched output $\hat{y}$ is the dimension of the reduced manifold $d$.
For the training phase, the data set is generated using several runs of the 1D-Burgers model with various viscosity parameters $\mu$, different basis sizes $d$ and the log of the Frobenius norms of the discrepancies between the full and the projected reduced
order solutions $log(\varepsilon)$. The machines will learn the sizes of reduced order basis $d$ associated with the parameter $\mu$ and the corresponding $log(\varepsilon)$. Later it will be able to estimate the proper size of reduced basis by providing it the specific viscosity parameter $\mu$ and the desired error $\log(\varepsilon)$. The computational cost is low once the probabilistic models are constructed. The output indicates the dimension of the reduced manifold for which the ROM solution satisfies the corresponding error threshold. Thus we do not need to compute the entire spectrum of the snapshots matrix in advance which for large spatial discretization meshes translates into important computational costs reduction.

For our experiments, approximately $9000$ samples were generated for different values of the viscosity parameter $\mu$ equally distributed within the interval $[10^{-2},1]$, various reduced basis dimensions from $4$ to $15$ and the corresponding $\log(\varepsilon)$. Figure \ref{fig:basis_contour_log} illustrates the contours of the log of reduced order model error, over the viscosity parameter domain and various POD sizes.

A neural network with $5$ hidden layers and hyperbolic tangent sigmoid activation function in each layer is used while for the Gaussian process we have used the squared-exponential-covariance kernel \eqref{eq_cov}. Table \ref{tab:Opt_log} show the average and variance of error in GP and ANN estimations using different sample sizes. The ANN outperforms the GP and as the number of data points grows, the accuracy increases and the variance decreases. The results are obtained using a conventional validation with $80\% $ of the sample size dedicated for training data and the other $20\% $ for the test data. The employed formula is described in equation \eqref{eqn:err_fold}.

\begin{figure}[h]
  \centering
  {\includegraphics[width=0.65\textwidth, height=0.55\textwidth]{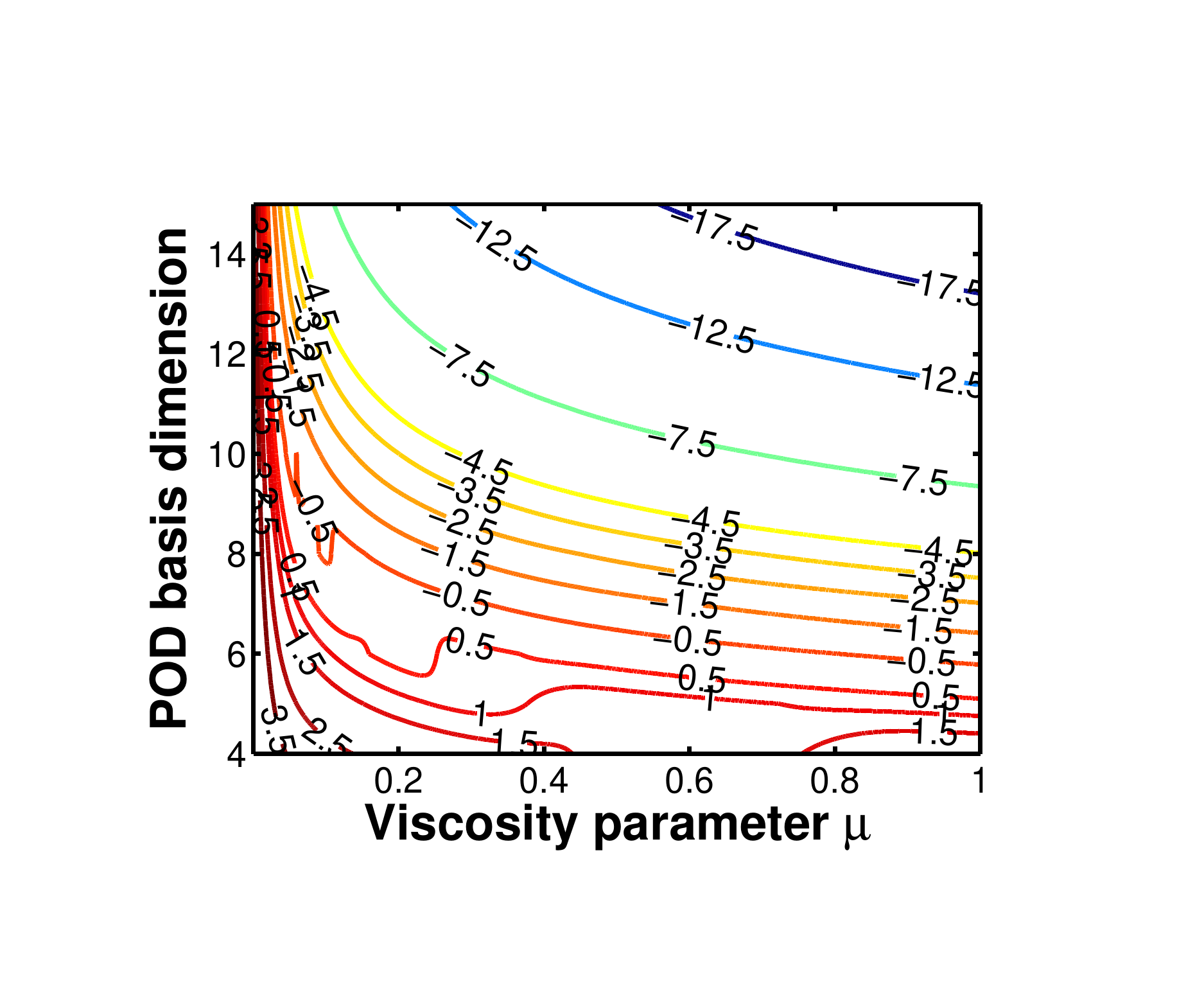}
\caption{Isocontours of the reduced model errors for different POD basis dimensions and viscosity parameters $\mu$.}
\label{fig:basis_contour_log} }
\end{figure}
%

%
\begin{table}[H]
\begin{center}
    \begin{tabular}{ | l | l | l |  l | l |}
    \hline
     & \multicolumn{2}{|c|}{GP} & \multicolumn{2}{|c|}{ANN} \\
 \hline
sample size &  $\textnormal{E}_{\rm fold}$   &  $\textnormal{VAR}_{\rm fold}$    & $\textnormal{E}_{\rm fold}$   &  $\textnormal{VAR}_{\rm fold}$
     \\ \hline
 100 & $ 0.2801 $ & $0.0901$ & $ 0.1580$ & $ 0.02204 $
     \\ \hline
 1000 & $0.1489$ & $ 0.0408 $ & $ 0.0121 $ & $ 0.0015 $
 \\ \hline
 3000 & $0.1013 $ & $ 0.0194 $ & $ 0.0273 $ & $ 0.0009 $
 \\ \hline
 5000 & $ 0.0884 $ & $ 0.0174 $ & $ 0.0080 $ & $ 0.0002 $
 \\ \hline
     \end{tabular}
\end{center}
 \caption{ Average and variance of errors in prediction of optimal reduced basis size using ANN and GP probabilistic models for different sample sizes}
  \label{tab:Opt_log}
\end{table}

Figures \ref{fig:hist_NNDim} and \ref{fig:hist_GPDim} shows the corresponding errors in estimation on $100$ and $1000$ training samples for both ANN and GP model.  These histograms as stated before, can assess the validity of GP assumptions. The data set distribution shape is closer to a Gaussian profile than in the case of the data set distribution discussed in section \ref{sect:err_estimate} used for generation of reduced order model error probabilistic models.

\begin{figure}[h]
  \centering
  \subfigure[$100$ samples] {\includegraphics[scale=0.4]{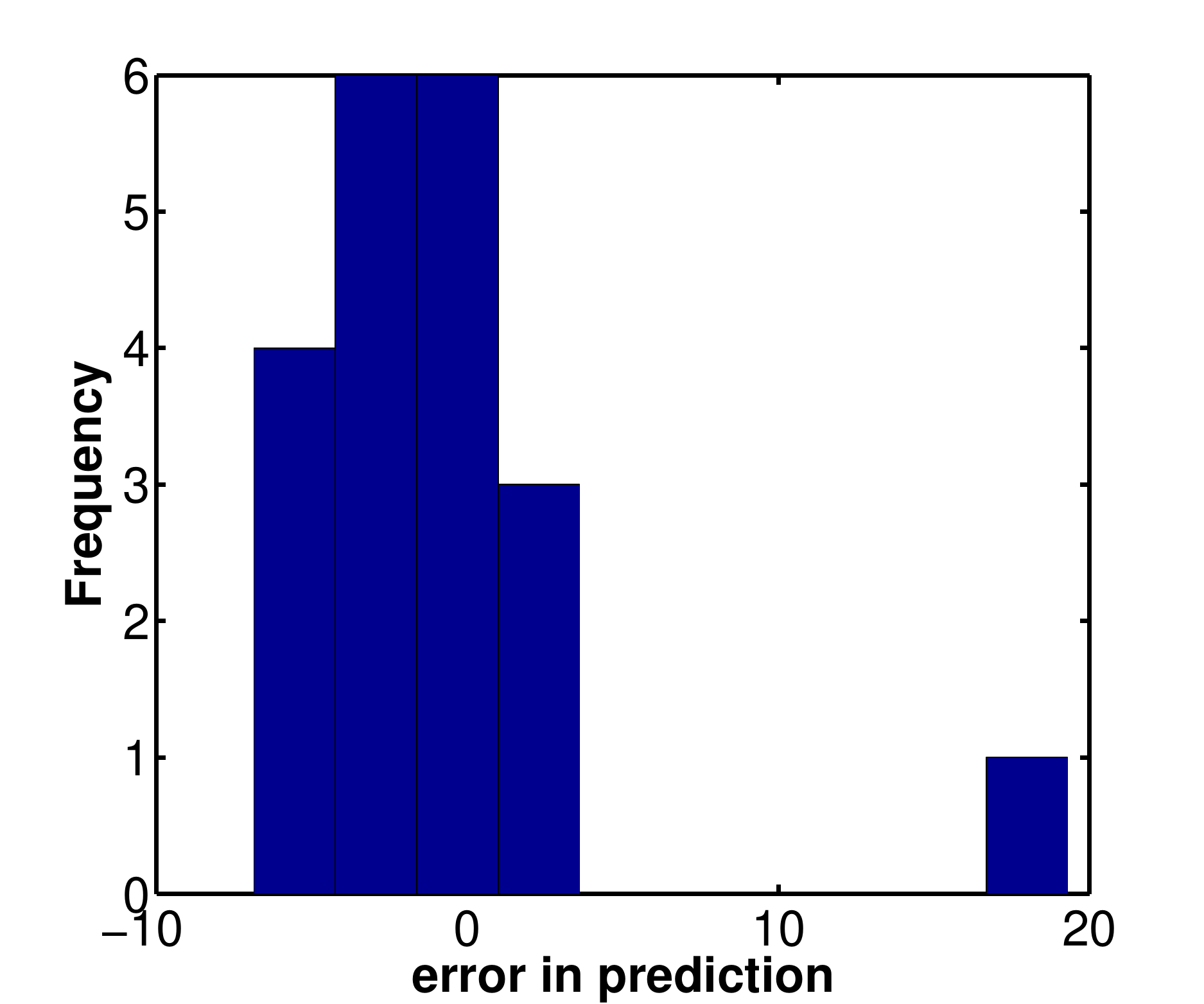}}
  \subfigure[$1000$ samples] {\includegraphics[scale=0.4]{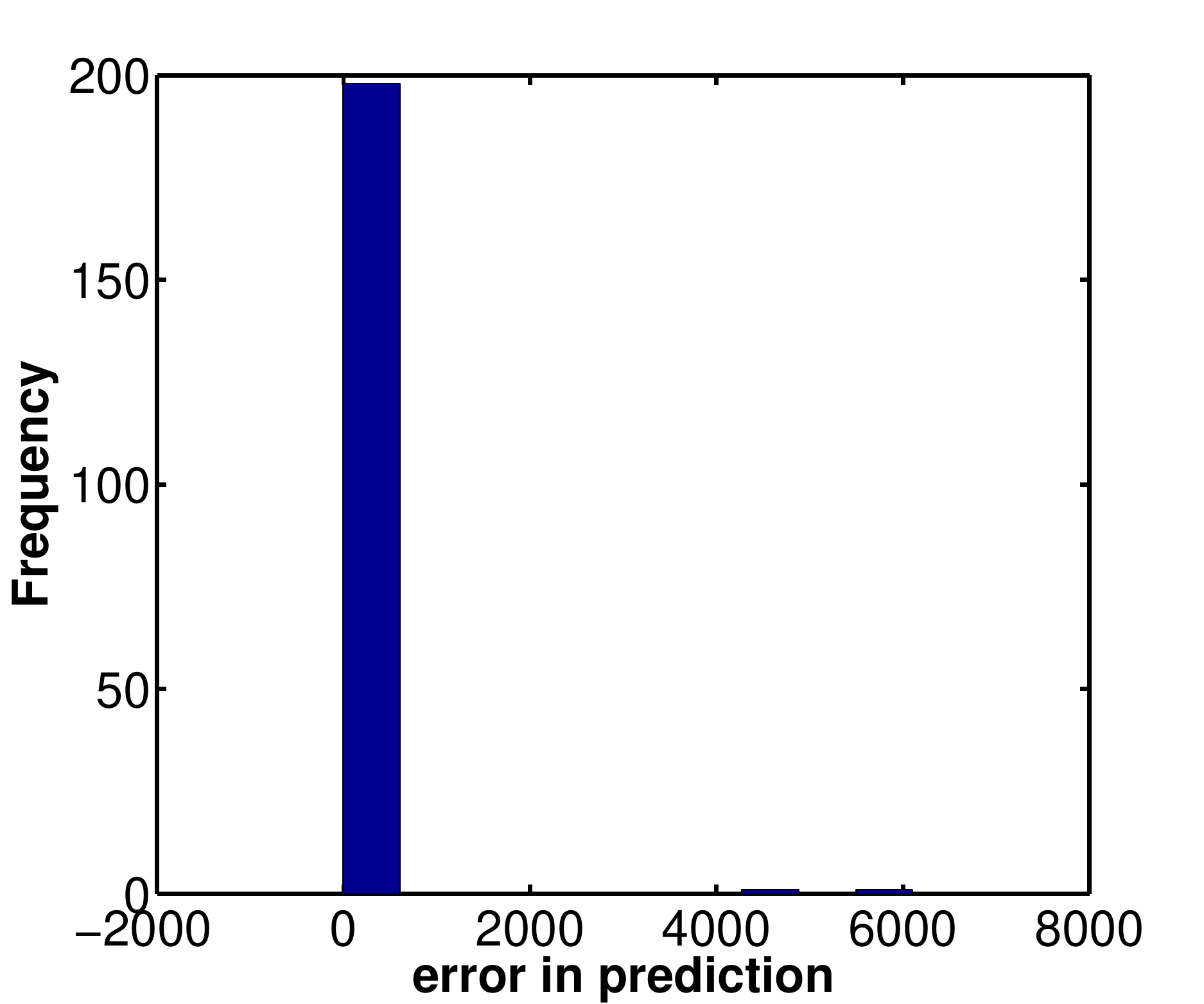}}
\caption{ Histogram of errors in prediction of the optimal reduced basis dimension using ANN for different sample sizes
\label{fig:hist_NNDim}}
\end{figure}
%


\begin{figure}[h]
  \centering
  \subfigure[$100$ samples] {\includegraphics[scale=0.4]{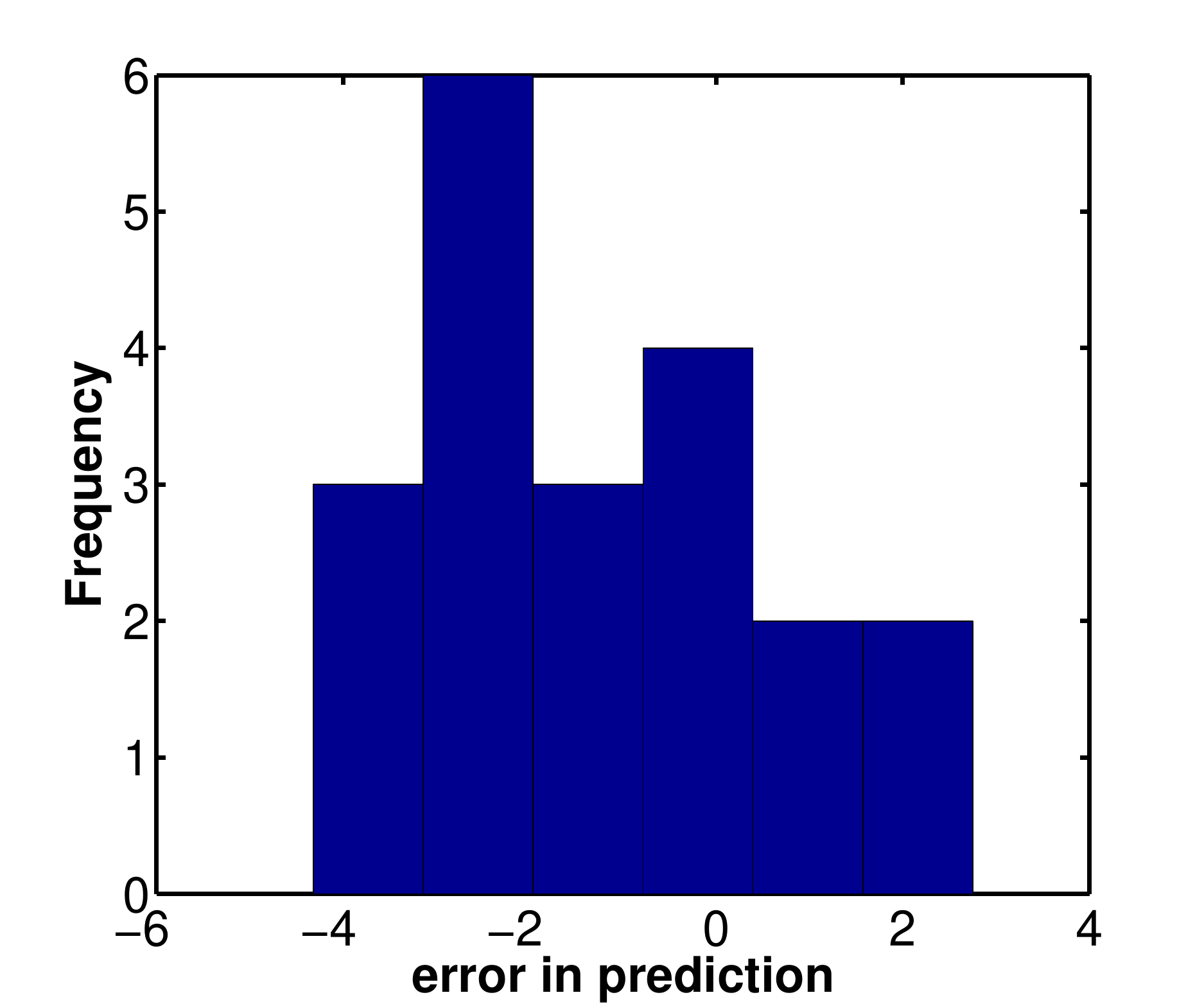}}
    \subfigure[$1000$ samples] {\includegraphics[scale=0.4]{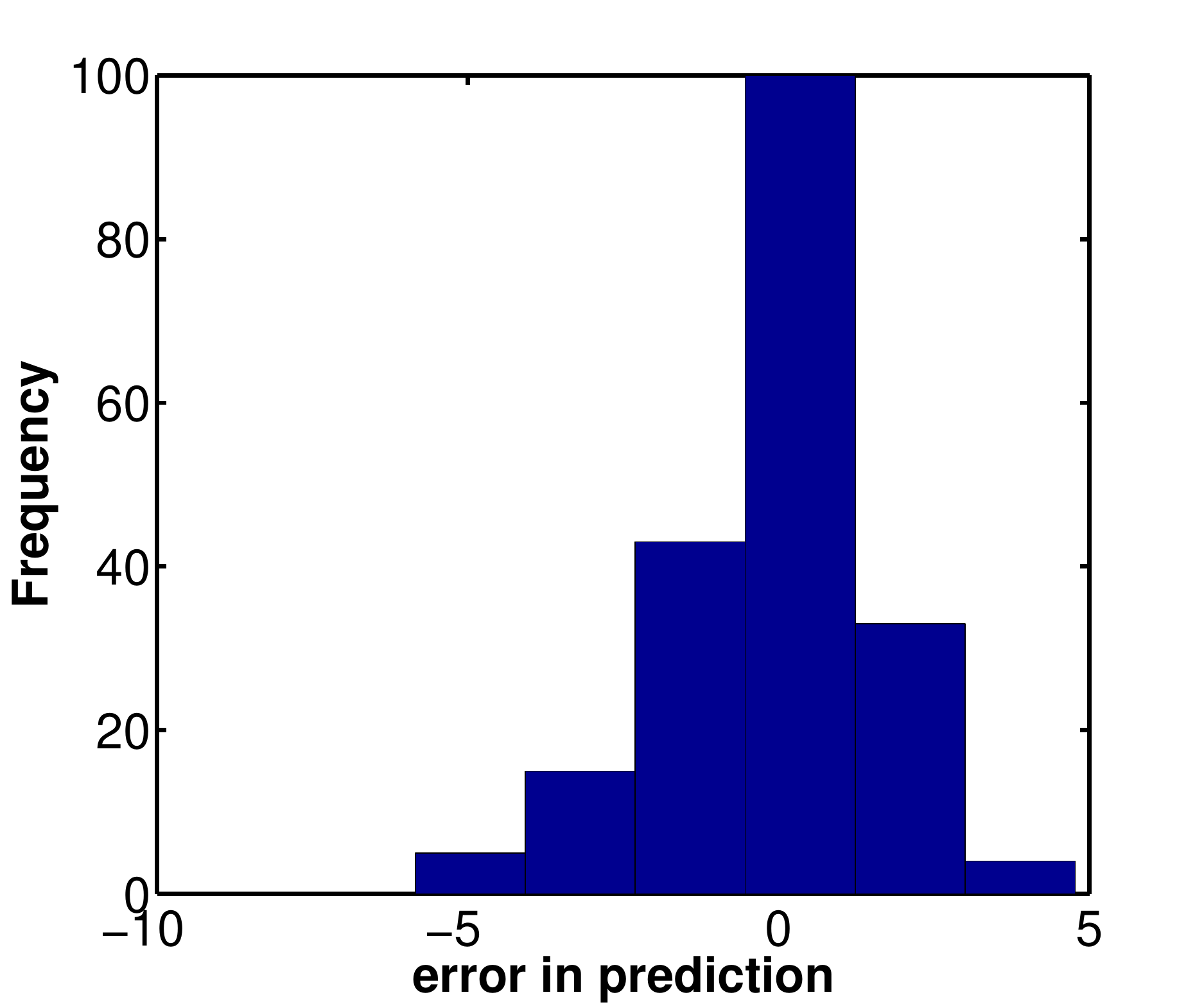}}
\caption{ Histogram of errors in prediction of the optimal reduced basis dimension using GP for different sample sizes
\label{fig:hist_GPDim}}
\end{figure}
%


To assess the accuracy of the probabilistic models, the data set is randomly partitioned into $5$ equal size sub-samples, and $5-$ fold cross-validation test is implemented. The $5$ results from the folds are averaged and they are presented in table \ref{tab:experm1}. The neural network model correctly estimated the size of the reduced manifold in $87\%$ cases.  Gaussian process correctly estimates the POD size $53\%$ of the times. The variance results shows that the GP model has more stable predictions indicating a higher bias in the outcomes.

\begin{table}[H]
\begin{center}
\begin{small}
    \begin{tabular}{ | c | p{1.55cm} | p{1.25cm} | c | c | c | p{1.55cm} |  p{2.75cm} |}
    \hline
  Dimension discrepancies& zero & one  & two  &  three  & four  & $>$ four & $ VAR $
      \\ \hline
 ANN & $87\% $ & $11\%$ & $2 \%$ & 0 & 0 & 0 & $2.779 \times 10^{-3}$
     \\ \hline
 GP & $53\%$ & $23 \%$ & $15\%$ & $5 \%$ &$ 3\%$ & $1 \%$ & $4.575 \times 10^{-4}$
     \\ \hline
    \end{tabular}
\end{small}
\end{center}
 \caption{POD basis dimension discrepancies between the ANN and GP predictions and true values over 5 fold cross validation. The errors variance is also computed.}
  \label{tab:experm1}
\end{table}

In figure \ref{fig:expm1_pod}, we compare the output of our probabilistic approaches against the eigenvalues estimation on a set of randomly selected test data. The eigenvalue estimation is the standard method for selecting the optimal reduced manifold dimension when a prescribed level of accuracy of the reduced solution is desired.  Here the desired accuracy $\varepsilon$ was set to $10^{-3}$. The mismatches between the predicted and true dimensions are depicted in the figure \ref{fig:expm1_pod}.  The predicted values are the averages over five different models where each model of ANN and GP are trained on random $80 \%$ split of dataset and tested on the fixed selected $20 \%$ test data. We notice that the snapshots matrix spectrum underestimates the true size of the manifold as expected since the errors due to integration in the reduced space are not accounted. The neural network predictions were extremely accurate for most of the samples  while the Gaussian process usually overestimated the reduced manifold dimensions.

%
\begin{figure}[H]
	\begin{centering}
	\includegraphics[width=0.55\textwidth, height=0.45\textwidth]{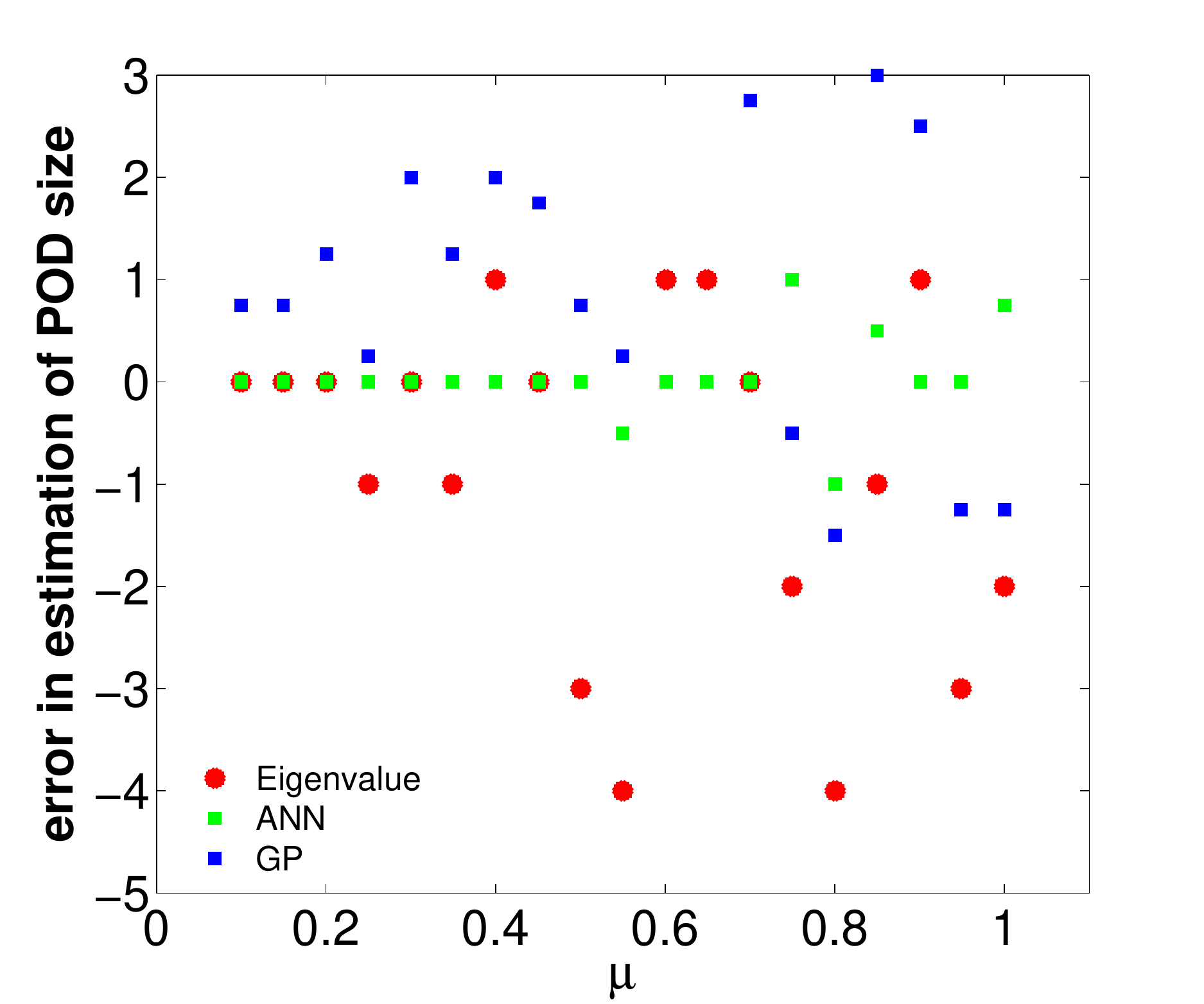}
    \caption{Average error of the POD size estimations on a randomly selected test data with desired accuracy of $\varepsilon=10^{-3}$.}
	\label{fig:expm1_pod}
	\end{centering}
\end{figure}
%


\section{Conclusions}
\label{sect:conc}

This work demonstrates the value of machine learning approaches to guide the construction of parametric space partitioning for the usage of efficient and accurate local reduced order models. While the current methodologies are defined in the sense of Voronoi tessellation \cite{du2002model} and rely on K-means algorithms \cite{amsallem2012nonlinear,Amsallem_Zahr_Washabaugh_2015,kaiser2014cluster}, our approach delimitates sub-regions of the parametric space by making use of Gaussian Processing and Artificial Neural Networks models for the errors of reduced order models and parametric domain sampling. The employed machine learning models differ from the one proposed in \cite{drohmann2015romes}  having more additional features such as reduced subspace dimension and are specially projected for accurate predictions of local reduced order models errors.  For each sub-region, an associated reduced order basis and operators are constructed depending on a single representative high-fidelity trajectory and, the corresponding local reduced order models solutions have known precision levels. The novel methodology is applied for a 1D-Burgers model, and a parametric map covering the viscosity domain with parametric sub-intervals and associated errors thresholds is designed.

Our numerical experiments revealed the non-monotonic property of the reduced order model error with respect to the distance between the current parametric configuration and the one used to generate the reduced subspace. Thus we proposed machine learning models for selecting a hierarchy of reduced bases producing the most accurate solutions for a new parameter configuration. Based on this hierarchy, three already existing methods involving bases interpolation and concatenation and high-fidelity model solutions interpolation are applied to enhance the quality of the associated reduced order model solutions.  It has been shown that the assumption of linear variation of the basis over the parametric space leads to a different reduced basis formulation than if the linear variation hypothesis of the high-fidelity solution over the parametric domain is followed. Several experiments were performed by scaling the time and space and modifying the nonlinear characteristics of the model. In most cases, interpolating the already existing high-fidelity trajectories generated the most accurate reduced order models for a new viscous parameter revealing that the solution behavior over the parametric region under study can be linearly approximated. Lagrange interpolation of bases in the tangent space of the Grassmann manifold and concatenation of bases for larger reduced subspaces showed also good performances.

Finally we addressed the problem of selecting the dimension of the reduced order model when its solution must satisfy a desired level of accuracy. Our approach based on learning better estimates the ROM basis dimension in comparison with the results obtained by truncating the spectrum of the snapshots matrix.

A future goal is to decrease the computational complexity of the parametric map design procedure. Currently the training data required by the probabilistic models relies on many high-fidelity simulations. By employing fast a-posteriori error estimation results \cite{nguyen2009reduced}, this dependency will be much decreased. In addition we plan to incorporate residual norm and rigorous error bounds among the data fitting models features to enhance their prediction capabilities as remarked in \cite{drohmann2015romes}.

\section*{Acknowledgements}
This work was supported in part and by the award NSF CCF 1218454 and by the Computational Science Laboratory at Virginia Tech.

\label{sect:bib}
\bibliographystyle{plain}

\bibliography{ML_bib,Additions_sandu,CDS_E_proposal,comprehensive_bibliography1,data_assim_fdvar,data_assim_weak-fdvar,NSF_KB,POD_bib,Razvan_bib,Razvan_bib_ROM_IP,Razvan_update_bib,reduced_models,ROM_state_of_the_art,sandu,Software}

\end{document}